\documentclass[conference]{IEEEtran}
\IEEEoverridecommandlockouts
\usepackage{cite}
\usepackage{graphicx}
\usepackage{balance}  
\usepackage{color}
\usepackage{amsfonts,amssymb,amsmath,bm}
\usepackage{makecell}
\usepackage{pifont}
\usepackage{stfloats}
\usepackage{subfigure}
\usepackage{multirow}
\usepackage[table,xcdraw]{xcolor}
\usepackage{graphicx}
\usepackage[ruled,vlined]{algorithm2e}
\usepackage{verbatim}
\usepackage{booktabs}
\usepackage{hyperref}
\usepackage{threeparttable}
\usepackage{bbding}
\usepackage{enumitem}
\usepackage[hyphenbreaks]{breakurl}

\usepackage{multirow}
\usepackage{graphicx}
\usepackage[normalem]{ulem}
\usepackage{graphicx}
\useunder{\uline}{\ul}{}

\usepackage{amsthm,bm}

\newtheorem{example}{Example}

\newtheorem{theorem}{Theorem}

\begin{document}


\title{\textsf{SparDL}: Distributed Deep Learning Training with Efficient Sparse Communication}


\author{Minjun Zhao$^{\dagger}$, Yichen Yin$^{\dagger}$, Yuren Mao$^{\dagger}$, Qing Liu$^{\dagger}$, Lu Chen$^{\dagger}$, Yunjun Gao$^{\dagger}$\\
\normalsize $^{\dagger}$\emph{College of Computer Science, Zhejiang University, Hangzhou, China}\\
\emph{$^{\dagger}$\{minjunzhao, yichenyin, yuren.mao, qingliucs, luchen, gaoyj\}@zju.edu.cn}\\
}

\maketitle

\begin{abstract}

Top-$\bm{k}$ sparsification has recently been widely used to reduce the communication volume in distributed deep learning. However, due to the Sparse Gradient Accumulation (SGA) dilemma, the performance of top-$\bm{k}$ sparsification  still has limitations. Recently, a few methods have been put forward to handle the SGA dilemma. Regrettably, even the state-of-the-art method suffers from several drawbacks, e.g., it relies on an inefficient communication algorithm and requires extra transmission steps. Motivated by the limitations of existing methods, we propose a novel efficient sparse communication framework, called \textsf{SparDL}. Specifically, \textsf{SparDL} uses the Spar-Reduce-Scatter algorithm, which is based on an efficient Reduce-Scatter model, to handle the SGA dilemma without additional communication operations. Besides, to further reduce the latency cost and improve the efficiency of \textsf{SparDL}, we propose the Spar-All-Gather algorithm. Moreover, we propose the global residual collection algorithm to ensure fast convergence of model training. Finally, extensive experiments are conducted to validate the superiority of \textsf{SparDL}.


\end{abstract}

\vspace{-1mm}
\section{Introduction}
\label{sec:intro}
\vspace{-1mm}

\subsection{Background}\label{sec:spardl_intro}

Nowadays, due to the large data volume and complex models, training deep learning models is becoming more and more time-consuming. To this end, many distributed model training techniques have been proposed for deep learning~\cite{DBLP:conf/sigmod/Zhang0SYJM19, DBLP:conf/sigmod/JiangFY018, DBLP:conf/sigmod/MiaoNSYJM021, DBLP:conf/sigmod/FardLLDB20, DBLP:journals/pvldb/MiaoZSNYTC21, DBLP:journals/vldb/GuoZJWZCL21, DBLP:conf/icde/ZhouLLOWY21}. Among them, data-parallel synchronous mini-batch stochastic gradient descent (S-SGD) is widely used~\cite{DBLP:conf/sigmod/MiaoNSYJM021, mikami2018massively, DBLP:conf/infocom/ShiC019}. Specifically, when the S-SGD is employed for distributed model training, at the end of each iteration, workers synchronize their local gradients, commonly by All-Reduce operation~\cite{DBLP:conf/ppopp/0002H22}, and update their training models with the same global gradients~\cite{DBLP:journals/vldb/JiangFYSC20}. 
However, S-SGD involves significant data communications~\cite{DBLP:conf/icdcs/ShiWZTWHC19, DBLP:conf/sc/RenggliAAAH19}.
Top-$k$ sparsification~\cite{DBLP:conf/iclr/LinHM0D18, DBLP:conf/nips/AlistarhH0KKR18, DBLP:conf/ppopp/0002H22} is one of the ways to reduce the communication overhead, which can sparsify the local gradients to about $1\%$ density while not impairing model convergence. Regrettably, top-$k$ sparsification sometimes suffers from the sparse gradient accumulation (SGA) dilemma. 

Specifically, when efficient All-Reduce methods (e.g., Rabenseifner's All-Reduce~\cite{DBLP:journals/ijhpca/ThakurRG05}) are used for synchronizing sparsified gradients, the selected gradients undergo multiple transmission and summation steps to obtain global gradients. Each summation increases the volume of sparse gradients for the following transmission process since the sparse gradients of different workers may come from differing indexes, resulting in the SGA dilemma. The SGA dilemma will lead to a quick increase in gradients, which may degrade to dense gradients. For example,  Figures~\ref{fig:1a} and~\ref{fig:1b} show the initial 3  gradients of workers 1 and 2, respectively. After worker 2 receives worker 1's gradients and adds it to its own gradients, worker 2 gets 5 gradients, as shown in Figure~\ref{fig:1c}. As transmission steps continue, the number of gradients for workers increases. The SGA dilemma significantly degrades the communication efficiency of distributed model training with sparsification. It is necessary to design an efficient method to alleviate the SGA dilemma.

\begin{figure}[t]
\centering
\hspace{-3mm}
\subfigure[The sparse gradients from worker 1]{
\includegraphics[width=0.12\textwidth]{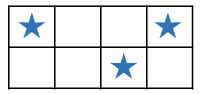}
\label{fig:1a}
}
\hspace{1.4mm}
\subfigure[The sparse gradients from worker 2]{
\includegraphics[width=0.12\textwidth]{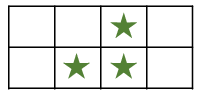}
\label{fig:1b}
}
\hspace{1.4mm}
\subfigure[The sparse gradients of worker 2 after summation]{
\includegraphics[width=0.12\textwidth]{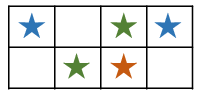}
\label{fig:1c}
}
\vspace{-3mm}
\caption{Illustration of the sparse gradient accumulation (SGA) dilemma}
\label{fig:issue}
\vspace{-6mm}
\end{figure}


\vspace{-2mm}
\subsection{State-of-the-art Methods}
\vspace{-1mm}
{

In the literature, some sparse All-Reduce methods have been proposed to handle the SGA dilemma, which is summarized in Table~\ref{tab:works} \cite{DBLP:journals/ijhpca/ThakurRG05, DBLP:conf/sc/RenggliAAAH19, DBLP:conf/icdcs/ShiWZTWHC19, DBLP:conf/ppopp/0002H22}. 
Specifically, Top$k$A~\cite{DBLP:conf/sc/RenggliAAAH19} alleviates the SGA dilemma by using the recursive doubling algorithm, which leads to the bandwidth cost proportional to the number of workers~\cite{DBLP:conf/icdcs/ShiWZTWHC19, DBLP:conf/ppopp/0002H22}. 
Top$k$DSA~\cite{DBLP:conf/sc/RenggliAAAH19} consists of Reduce-Scatter and All-Gather operations. In the Reduce-Scatter operation, Top$k$DSA addresses the SGA dilemma by directly sending gradients, which also causes high latency cost. In the All-Gather operation, Top$k$DSA allows the SGA dilemma to happen and switches to dense transmission when the sparse gradient is large enough, leading to high bandwidth cost.
gTop$k$~\cite{DBLP:conf/icdcs/ShiWZTWHC19, DBLP:conf/ijcai/ShiZWTC19} utilizes reduction tree and broadcast tree models, both of which have high bandwidth cost, to solve the SGA dilemma. 

O$k$-Top$k$~\cite{DBLP:conf/ppopp/0002H22} is the state-of-the-art method among four methods, which improves Top$k$DSA with two compression steps. However, O$k$-Top$k$ also suffers from some limitations.
First, O$k$-Top$k$ directly sends gradients to the target workers in the Reduce-Scatter operation, which has high latency cost.
Second, in the All-Gather operation, O$k$-Top$k$ employs several extra communication operations to balance the uneven distribution of sparse gradient among workers, causing high latency cost and large upper bound of bandwidth cost for the All-Gather operation.
Third, as shown in Table~\ref{tab:works}, the upper bound of O$k$-Top$k$'s bandwidth cost is $6\tfrac{P-1}{P}k\beta$. However, in reality, the bandwidth cost of O$k$-Top$k$ may be higher than $6\tfrac{P-1}{P}k\beta$. The reason is two-folded. (i) The local sparse gradients from each worker are balanced every 64 iterations, which causes uneven distribution before the next balancing and leads to different communication volumes for different workers in the Reduce-Scatter operation. (ii) O$k$-Top$k$ utilizes threshold pruning, which may select more than $k$ sparse gradients.

Therefore, the performance of sparse All-Reduce still has room for improvement.

}


\begin{table}[t]
\small
\caption{Communication complexity of sparse All-Reduce methods ($P$ is the number of workers, $n$ is the number of dense gradients, $k$ is the number of sparse gradients ($k \ll n$) and $d$ is the number of teams in \textsf{SparDL})}
\vspace{-2mm}
\centering
\begin{tabular}{|l|l|l|}
\hline
\textbf{Algorithms} & \textbf{Latency Cost} & \textbf{Bandwidth Cost}\\ \hline
Top$k$A~\cite{DBLP:conf/sc/RenggliAAAH19} & $\log P\alpha$ & $2(P-1)k\beta$\\ \hline
Top$k$DSA~\cite{DBLP:conf/sc/RenggliAAAH19} & \begin{tabular}[c]{@{}l@{}}$(P+2\log P) \alpha$\end{tabular} & \begin{tabular}[c]{@{}l@{}}$[4\tfrac{P-1}{P}k\beta, \tfrac{P-1}{P}(2k+n)\beta]$\end{tabular}\\ \hline
gTop$k$~\cite{DBLP:conf/icdcs/ShiWZTWHC19} & $2\log P\alpha$ & $4\log Pk\beta$\\ \hline
O$k$-Top$k$~\cite{DBLP:conf/ppopp/0002H22} & \begin{tabular}[c]{@{}l@{}}$2(P+\log P) \alpha$\end{tabular} & \begin{tabular}[c]{@{}l@{}}$[2\tfrac{P-1}{P}k\beta, 6\tfrac{P-1}{P}k\beta]$\end{tabular}\\ \hline
\begin{tabular}[c]{@{}l@{}}\textbf{\textsf{SparDL}}\end{tabular} & $2\log P\alpha$ & $4\tfrac{P-1}{P}k\beta$\\ \hline
\begin{tabular}[c]{@{}l@{}}\textbf{\textsf{SparDL}}\\ \textbf{(R-SAG)}\end{tabular} & \begin{tabular}[c]{@{}l@{}}$(2\log\tfrac{P}{d}$\\ $+\log d)\alpha$\end{tabular} & \begin{tabular}[c]{@{}l@{}}$2(\tfrac{2P-2d}{P}+\tfrac{d}{P}\log d)k\beta$\end{tabular} \\ \hline
\begin{tabular}[c]{@{}l@{}}\textbf{\textsf{SparDL}}\\ \textbf{(B-SAG)}\end{tabular} & \begin{tabular}[c]{@{}l@{}}$(2\log\tfrac{P}{d}$\\ $+\log_2 d)\alpha$\end{tabular} & \begin{tabular}[c]{@{}l@{}}$[2\tfrac{d^2+P-2d}{Pd}k\beta$,\\$2\tfrac{d^2+2P-3d}{P}k\beta]$\end{tabular}\\ \hline
\end{tabular}
\label{tab:works}
\vspace{-5mm}
\end{table}


\vspace{-2mm}
\subsection{Our Solutions}
\vspace{-1mm}
Motivated by the limitations of existing methods, in this paper, we propose a novel sparse communication framework, called \textsf{SparDL}, for distributed deep learning training. \textsf{SparDL} incorperates three novel algorithms, i.e., \textit{Spar-Reduce-Scatter}, \textit{global residual collection}, and \textit{Spar-All-Gather}. To be specific, first, \textsf{SparDL} evenly divides all workers into $d$ teams. Second, \textsf{SparDL} utilizes \textit{Spar-Reduce-Scatter} to Reduce-Scatter sparse gradients for each team. Then, \textsf{SparDL} uses \textit{Spar-All-Gather} to synchronize gradients among $d$ teams. Finally, \textsf{SparDL} adopts Bruck All-Gather~\cite{DBLP:journals/tpds/BruckHKUW97} to synchronize gradients for each team. In both \textit{Spar-Reduce-Scatter} and \textit{Spar-All-Gather},  \textsf{SparDL} employs \textit{global residual collection} to collect the discarded gradients. The details of \textit{Spar-Reduce-Scatter}, \textit{global residual collection}, and \textit{Spar-All-Gather} are as follows.

\begin{itemize}
[leftmargin=*]
\vspace{-0.5mm}
\item \textit{Spar-Reduce-Scatter} is proposed to efficiently Reduce-Scatter sparse gradients for each team. To this end, we devise a non-recursive Reduce-Scatter structure for \textit{Spar-Reduce-Scatter} to ensure the efficiency of \textsf{SparDL} even for the arbitrary number of workers. Thus, \textit{Spar-Reduce-Scatter} can partition the gradients into blocks and apply block top-$k$ selection on every gradient block before each transmission step to maintain the volume of sparse gradients. With the help of \textit{Spar-Reduce-Scatter}, \textsf{SparDL} does not require extra transmission steps to balance the gradients among workers and avoids resorting to unstable threshold pruning to achieve global compression.
\item \textit{Global residual collection} is devised to ensure fast convergence of the training process. For maintaining the volume of sparse gradients, \textsf{SparDL} utilizes multiple top-$k$ selections, which may discard important gradients and result in slower convergence rate. To keep the fast convergence rate of the training process, we propose the \textit{global residual collection} algorithm to collect and reuse all the discarded gradients.
\item \textit{Spar-All-Gather} is based on the divide-and-conquer strategy and is proposed to reduce latency cost and improve the efficiency of \textsf{SparDL} further. Specifically, \textit{Spar-All-Gather} divides workers into $d$ teams. By adjusting the number of teams, we can adjust the ratio of latency and bandwidth in the total communication complexity. For different situations, we use two versions of \textit{Spar-All-Gather}, i.e., R-SAG and B-SAG. R-SAG is used when the number of teams is $2^i$ ($i \in \mathbf{N^*}$). Otherwise, we employ B-SAG.

\end{itemize}

{
 \noindent \textbf{Contributions.} This paper makes the following contributions. 
\begin{itemize}
[leftmargin=*]
\item We present a novel sparse communication framework, called \textsf{SparDL}, to handle the SGA dilemma for efficient distributed deep learning. 
\item We propose three novel algorithms, i.e, \textit{Spar-Reduce-Scatter}, \textit{global residual collection}, and \textit{Spar-All-Gather}, and integrate these three algorithms seamlessly into \textsf{SparDL}. 
\item We conduct extensive experiments on seven different distributed deep learning cases to verify the effectiveness and efficiency of \textsf{SparDL}. 
Experimental results show that \textsf{SparDL} is up to 4.9$\times$ faster than the state-of-the-art methods, while maintaining comparable effectiveness.
\end{itemize}}

\noindent
\textbf{Outline.}
{The rest of this paper is organized as follows.
Section~\ref{sec:problem} introduces preliminaries.
Section~\ref{sec:method} presents the framework of \textsf{SparDL}, and the details of \textit{Spar-Reduce-Scatter}, \textit{global residual collection}, and \textit{Spar-All-Gather}.
Section~\ref{sec:exe} reports experimental results.
Section~\ref{sec:related} reviews the related work.
Section~\ref{sec:lim} discuss the limitations and practical implications of \textsf{SparDL}.
Finally, we conclude the paper in Section~\ref{sec:conclusion}. }

\vspace{-1mm}
\section{Preliminaries}
\label{sec:problem}
\vspace{-1mm}


\noindent\textbf{Cost model for communication complexity.}
The aforementioned latency ($\alpha$)-bandwidth ($\beta$) cost model ($\alpha$-$\beta$ model)~\cite{DBLP:conf/imw/SarvothamRB01, DBLP:journals/pc/Hockney94} is widely used to analyze the communication complexity~\cite{DBLP:conf/ipps/Pjesivac-GrbovicABFGD05, DBLP:conf/ppopp/0002H22}. In this cost model, $\alpha$ is the latency cost unit of a transmission between two workers. $\beta$ is the bandwidth cost unit.
All communication cost can be modeled by $x\alpha+y\beta$ where $x$ can be counted as the number of transmission rounds and $y$ can be counted as the total volume of data received by a worker during the communication process~\cite{DBLP:journals/concurrency/ChanHPG07}. 


\vspace{0.025in}
\noindent\textbf{All-Reduce, Reduce-Scatter and All-Gather.}
All-Reduce, Reduce-Scatter and All-Gather are common communication operations. 
Figure~\ref{fig:allreduce} shows the results after these operations. 
All-Reduce operation is commonly used to synchronize and sum gradients from all workers~\cite{DBLP:conf/sc/SensiGA0H21, DBLP:conf/icdm/HakimiALS21, DBLP:journals/concurrency/ChanHPG07, DBLP:conf/hpca/DongCZYWFZLSPGJ20, DBLP:conf/icpp/ChuLASHEP17, DBLP:journals/pvldb/LiZVSNLPSVDC20}.
Besides, using Reduce-Scatter and All-Gather successively can achieve the same results as All-Reduce. Reduce-Scatter reduces each part of gradients into its corresponding worker. And All-Gather gathers all parts of gradients from all workers in each worker. There are two popular, efficient All-Gather algorithms called recursive doubling~\cite{DBLP:journals/ijhpca/ThakurRG05, DBLP:journals/concurrency/ChanHPG07} and Bruck All-Gather~\cite{DBLP:journals/tpds/BruckHKUW97}, as shown in Figure~\ref{fig:reduceallgather}. 
The recursive doubling algorithm is efficient when the number of workers is a power of 2, but cannot be used directly when the number is not a power of 2~\cite{DBLP:journals/ijhpca/ThakurRG05, DBLP:journals/concurrency/ChanHPG07}. In contrast, Bruck All-Gather is efficient and reach the lower bound of bandwidth cost at any number of workers~\cite{DBLP:journals/ijhpca/ThakurRG05}.
For recursive doubling, each worker communicates and exchanges all data with a worker at a distance of $2^t$ at the $t$-th step. And for Bruck All-Gather, each worker sends data to a worker at a distance of $2^t$ on one side and receives data from a worker at a distance of $2^t$ on the other side at the $t$-th step. The communication complexity of recursive doubling and Bruck All-Gather, when the number of workers is a power of 2, is
\vspace{-2mm}
\begin{equation}
\vspace{-2mm}
T_{All-Gather} = \log_2{P}\alpha + n\tfrac{P-1}{P}\beta.
\end{equation}


\begin{figure}[t]
\centering
\includegraphics[width=0.48\textwidth]{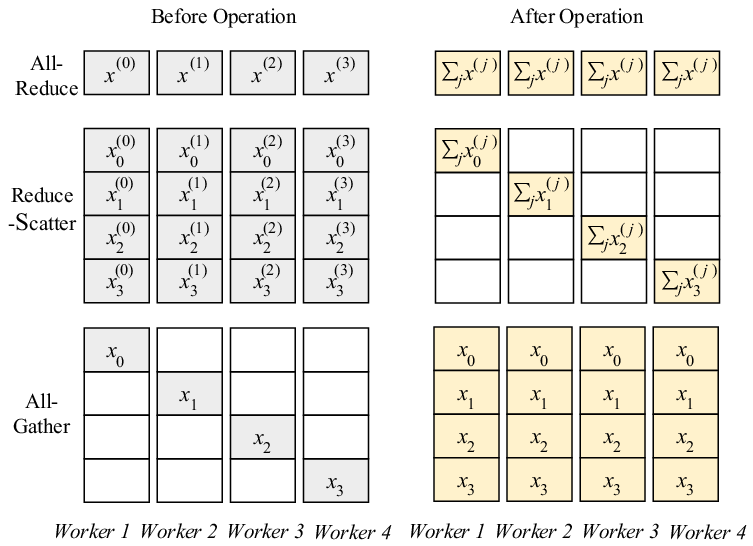}
\vspace{-3mm}
\caption{Illustration of All-Reduce, Reduce-Scatter and All-Gather operation}
\vspace{-6mm}
\label{fig:allreduce}
\end{figure}

\vspace{0.025in}
\noindent\textbf{Top-$\bm{k}$ sparsification.}
Top-$k$ sparsification is usually used in efficient distributed deep learning for sparsifying gradients~\cite{DBLP:conf/iclr/LinHM0D18, DBLP:conf/nips/AlistarhH0KKR18, DBLP:conf/ppopp/0002H22}. It selects $k$ gradients with the largest absolute values from dense gradients where $k$ is usually set to a certain proportion multiplied by the number of dense gradients.
{Top-$k$ sparsification does not seriously affect learning performance, which has been proved by many previous works theoretically~\cite{DBLP:conf/nips/AlistarhH0KKR18, DBLP:conf/nips/StichCJ18, DBLP:conf/ijcai/ShiZWTC19} and empirically~\cite{DBLP:conf/emnlp/AjiH17, DBLP:conf/iclr/LinHM0D18, DBLP:conf/ppopp/0002H22, DBLP:conf/sc/RenggliAAAH19}.
}

\begin{figure}[tb]
	\centering
	\subfigure[{Recursive doubling for All-Gather}]{
    	
    \centering
    \hspace{-5mm}
	\includegraphics[width=0.5\textwidth]{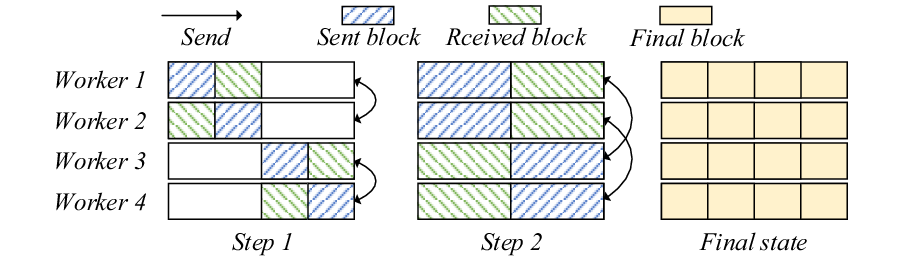}}
 
    \vspace{-4mm}
	\subfigure[{Bruck All-Gather for All-Gather}]{
    \hspace{-5mm}
	\includegraphics[width=0.5\textwidth]{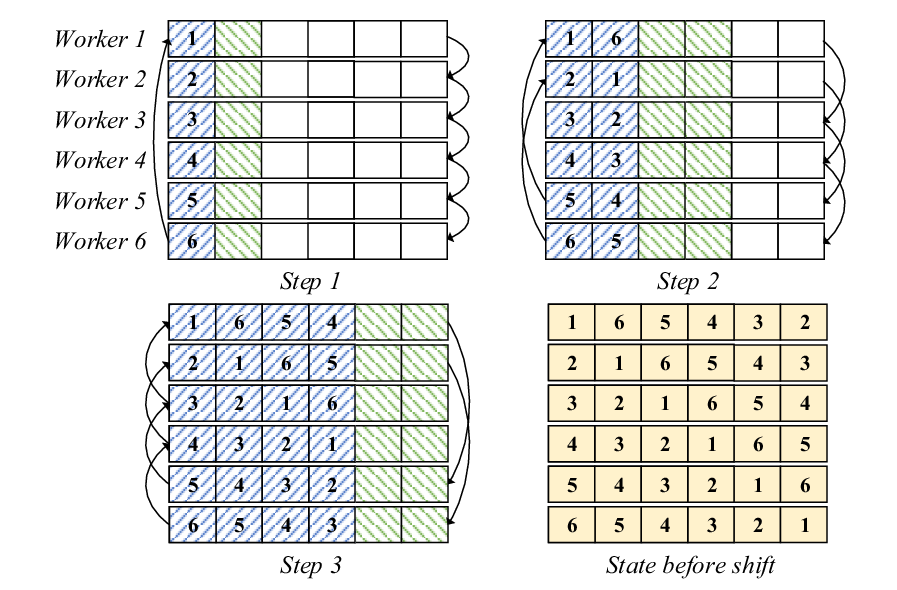}}
	\vspace{-5mm}
	\caption{Illustration of recursive doubling and Bruck All-Gather}
	\vspace{-6mm}
	\label{fig:reduceallgather}
\end{figure}
\vspace{-1mm}
\section{The Proposed \textsf{SparDL}}
\label{sec:method}

\begin{figure*}[t]
\centering
\includegraphics[width=1.0\textwidth]{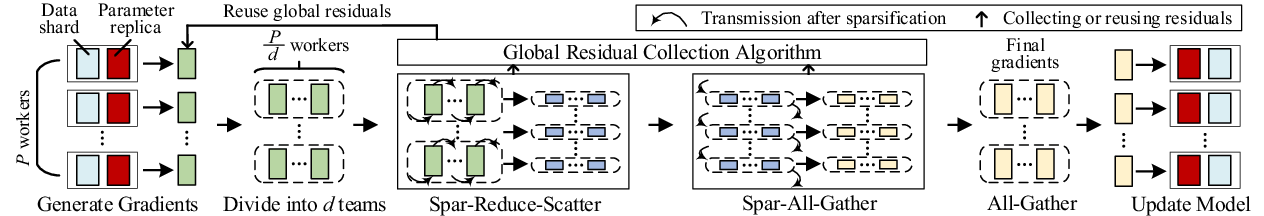}
\vspace{-7mm}
\caption{An overview of \textsf{SparDL} framework}
\label{fig:framework}
\vspace{-5mm}
\end{figure*}

{In this section, we describe the proposed \textsf{SparDL}, an efficient distributed deep learning framework that solves the SGA dilemma for sparse gradients.
\textsf{SparDL} is composed of three proposed algorithms: \textit{Spar-Reduce-Scatter}, \textit{global residual collection} and \textit{Spar-All-Gather}. 
In the following sections, we first overview the framework of \textsf{SparDL}, and then provide an in-depth introduction to these algorithms.}

\subsection{Overview}
\vspace{-1mm}
The overview of the \textsf{SparDL} framework is shown in Fig.~\ref{fig:framework}. Specifically, for a cluster with $P$ workers, we first generate gradients at each worker from forward propagation and backward propagation in parallel. Next, we plus the locally stored residuals from the global residual collection algorithm to the local gradients. Then we divide all $P$ workers into $d$ teams equally. We use the \textit{Spar-Reduce-Scatter} algorithm, \textit{Spar-All-Gather} algorithm and an All-Gather operation successively to get the final gradients after synchronization. If the user sets $d=1$, only \textit{Spar-Reduce-Scatter} and an All-Gather operation are used.
Meanwhile, the \textit{global residual collection} algorithm collects discard gradients as residuals during the \textit{Spar-Reduce-Scatter} and \textit{Spar-All-Gather} algorithms at each worker. At last, we update the model replica with the final gradients at each worker.
Besides, unless noted otherwise, \textsf{SparDL} denotes \textsf{SparDL} with $d=1$ and thereby only utilizes \textit{Spar-Reduce-Scatter} and an All-Gather operation.

\vspace{-2mm}
\subsection{Spar-Reduce-Scatter Algorithm}
\vspace{-1mm}
{
To address the SGA dilemma with low communication complexity, we propose \textit{Spar-Reduce-Scatter} (SRS) algorithm, which combines the Reduce-Scatter phase with multiple block-wise sparsification processes for the first time. Moreover, we propose a non-recursive structure for Reduce-Scatter to ensure low communication complexity and efficiently work on any number of workers without additional operations. A worker's source and target workers differ at one transmission step in this structure. A six-worker example is provided to illustrate SRS, as shown in Figure~\ref{fig:Spar-Reduce-Scatter}.}
Besides, to clarify the SRS algorithm, we set $d=1$ and separate SRS into two processes: partitioning and transmission with sparsification.

\vspace{0.025in}
\noindent\textbf{1) Partitioning.} For worker $r^w$, which is the $w$-th worker of all $P$ workers, given gradients $G^{(w, t)}$ from the latest iteration $t$ after training one batch (forward propagation and backward propagation) and residuals $\xi^{t-1}_w$ from last iteration $t-1$.
First, worker $r^w$ sums up $G^{(w, t)}$ and $\xi^{t-1}_w$ as new $G^{(w, t)}$. Then it partitions $G^{(w, t)}$ into $P$ blocks $\{{G^{(w, t)}_0} , \ldots, {G^{(w, t)}_P}\}$ and sparsifies its local dense gradients by selecting top-$\tfrac{k}{P}$ values of each block. Then we partition $P$ blocks into one preservation bag $B_0$ and $l = \lceil \log_2P \rceil$ sending bags $\{B_{1}^w, \ldots, B_{l}^w\}$ in sequence. These sending bags will be used as sending units in the following transmission steps, and we will preserve the preservation bag. Specifically, we start from the block ${G^{(w, t)}_w}$ and put it into the preservation bag $B_0^w$. Then we put the next $2^0$ blocks from $w+2^0$ to $w+2^1-1$ into sending bag $B_1^w$, and the following $2^1$ blocks from $w+2^1$ to $w+2^2-1$ into $B_2^w$. Repeat this process until all the blocks are put into different bags. Note that, in the partitioning process, $P$ blocks are considered to line up in a circle, i.e., if reaching block $P$ in the partitioning process, continue putting blocks from block $1$. It is possible that the remaining blocks are not enough to fill the last sending bag $B_{l}$, and the remaining number will be $E = P-2^{\lceil \log_2 P \rceil -1}$.

\begin{example}
\vspace{-2mm}
Figure~\ref{fig:Spar-Reduce-Scatter} presents the partitioning process. Suppose we have 6 workers. Take worker 1 for example, we put its first block $\{1\}$ into preservation bag $B_0^1$, and $l = \lceil \log_26 \rceil = 3$, its next $2^0$ block $\{2\}$ into sending bag $B_1^1$ and the next $2^1$ blocks $\{3, 4\}$ into sending bag $B_2^1$. The last sending bag $B_3^1$ should be put $2^2$ blocks, but worker 1 has only $E=6-2^{\lceil \log_26 \rceil -1}=2$ blocks $\{5, 6\}$ left. Therefore, we put block $\{5, 6\}$ into sending bag $B_3^1$, which is not full.
\end{example}

\vspace{-2mm}
\vspace{0.025in}
\noindent\textbf{2) Transmission with sparsification.} 
After partitioning blocks into bags, we transmit all sending bags of blocks from the last bag $B_{l}^w$ to $B_1^w$ in sequence. 
Specifically, for worker $r^w$ at each step $i$, it regards worker $r^{w+2^{l-i}}$ as target and regards worker $r^{w-2^{l-i}}$ as source. It sends bag $B_{l-i+1}^w$ to worker $r^{w+2^{l-i}}$ and receives bag $B_{l-i+1}^{w-2^{l-i}}$ from worker $r^{w-2^{l-i}}$. 
After receiving a sending bag from the source, each worker adds the sparse gradients from the blocks in the received sending bag to the gradients in the locally remaining blocks according to the gradient indexes.
Then we utilize block-wise sparsification in each worker, i.e., selecting top-$\tfrac{k}{P}$ values in each locally remaining block after the summation and removing the unselected gradients.
Through this sparsification in each step, the number of gradients of each block in sending bag is maintained as $\tfrac{k}{P}$, which solves the SGA dilemma. At each step, the received sending bag is always a subset of the held blocks, i.e., blocks that have not been sent, according to Theorem~\ref{the:receive}. Thus, every worker's number of held blocks decreases at every step. After $l$ steps, all sending bags have been sent. And each worker only holds one block with a different rank related to the worker's rank at the preservation bag, e.g., $r^w$ holds block ${G^{(w, t)}_w}$, which is in line with the purpose of Reduce-Scatter.
Besides, this process works in a non-recursive way, where each worker's source and target can be different at each transmission step.
Thus, the SRS algorithm is available on any number of workers directly, without additional operations.

\begin{figure*}[t]
	\centering
	\includegraphics[width=0.95\textwidth]{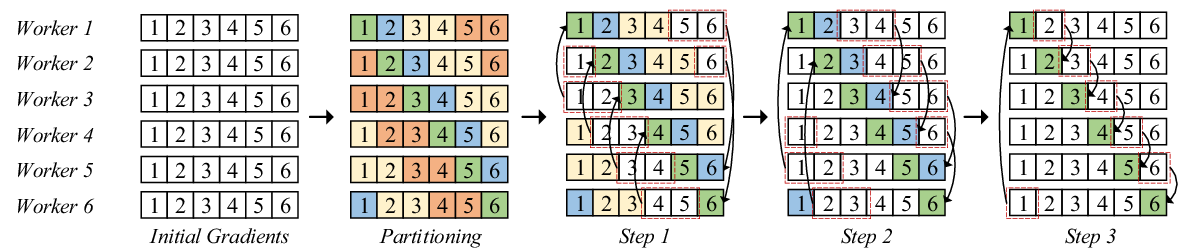}
	\vspace{-3mm}
	\caption{Spar-Reduce-Scatter Algorithm. The number in the block represents the position of the block. Each block contains part of gradients (dense or sparse) of its corresponding position.}
	\label{fig:Spar-Reduce-Scatter}
 	\vspace{-6mm}
\end{figure*} 

\begin{example}
\vspace{-2mm}
Figure~\ref{fig:Spar-Reduce-Scatter} also presents the transmission and sparsification process. Take worker 1 for example, at step 1, the communication distance is $2^{3-1} = 4$. Thus, it sends bag $B_3$ with 2 blocks to worker 5 and receives a bag with 2 blocks from worker 2. At step 2, the communication distance is $2^{3-2} = 2$. Thus, worker 1 sends bag $B_2$ with 2 blocks to worker $3$ and receives a bag with 2 blocks from worker $5$. At the last step, worker 1 sends bag $B_1$ with 1 block to the worker $2$ and receives a bag with 1 block from the worker $6$.
\end{example}

\vspace{-4mm}
\begin{theorem}
At each step $i$ in transmission, the ranks of blocks in sending bag from the $w$-th worker are a subset of those of the blocks held by the $w+2^{l-i}$-th worker.
\label{the:receive}
\end{theorem}

\vspace{-2mm}
\begin{proof}
\vspace{-2mm}
The left side of sending blocks from $w$-th worker at step $i$ is $w+2^{l-i}$. Besides, the blocks held by the $w+2^{l-i}$-th worker is from $w+2^{l-i}$ to $w+2^{l-i}+2^{l-i}-1$. Thus, at step $i$, the left side of sending bag is also the left side of reserving bag. Since the size of the last sending bag is $E = P-2^{\lceil \log_2 P \rceil -1}$ and $E<P-E$, the right side of the last sending bag will be in the remaining blocks of the $w+2^{l-i}$-th worker. Thus the Theorem holds at step 1. Since the range of other sending blocks from $w$-th worker at step $i$ is from $w+2^{l-i}$ to $w+2^{l-i+1}-1$ and the $w+2^{l-i}$-th worker holds block from $w+2^{l-i}$ to $w+2^{l-i}+2^{l-i}-1$ at step $i$, the theorem holds at step $i, i \neq 1$. Therefore, the Theorem holds at each step.
\vspace{-2mm}
\end{proof}

After SRS, each worker holds a block of $\tfrac{k}{P}$ sparse gradients at a different position. At last, we use the All-Gather operation on the block of each worker, which sends the local data of the worker to all other workers. We choose the Bruck algorithm~\cite{DBLP:journals/tpds/BruckHKUW97} for the All-Gather operation (mentioned in Section 2 and shown in Figure~\ref{fig:reduceallgather}) since it also efficiently works on any number of workers without additional operations. After Bruck All-Gather, all workers have the same global gradients, which means synchronization is complete.
Therefore, \textsf{SparDL} is efficient for every number of workers because both SRS and Bruck algorithms are efficient for every number of workers.

\vspace{0.025in}
\noindent\textbf{Communication Complexity Analysis.} 
In the transmission and sparsification process, each worker takes $l = \lceil \log_2P \rceil$ rounds of communication in parallel. Thus, the latency of SRS is $\lceil \log_2P \rceil\alpha$. During the transmission process, each worker sends $P-1$ blocks in sending bags. Since sparse gradients should be stored with indices and values with the commonly used form, i.e., coordinate (COO) format, the transmission volume of each block in sending bag is $\tfrac{2k}{P}$. Hence, the bandwidth of transmission in SRS is $2k\tfrac{P-1}{P}\beta$. The communication complexity of SRS is 
\vspace{-2mm}
\begin{equation}
\vspace{-2mm}
T_{1} = \lceil \log_2P \rceil\alpha + 2k\tfrac{P-1}{P}\beta. 
\end{equation}
The volume of sparse gradients is $2\tfrac{k}{P}$ (indexes and values) in each worker. Thus, the communication complexity of Bruck algorithm is 
\vspace{-2mm}
\begin{equation}
\vspace{-2mm}
T_{2} = \lceil \log_2P \rceil\alpha + 2k\tfrac{P-1}{P}\beta. 
\end{equation}
Therefore, the total communication complexity of \textsf{SparDL} is
\vspace{-2mm}
\begin{equation}
\vspace{-2mm}
T_{all} = T_1+T_2=2\lceil \log_2P \rceil\alpha + 4k\tfrac{P-1}{P}\beta.
\end{equation}

\vspace{0.025in}
\noindent\textbf{Optimization for SRS.} 
After observation, We find it unnecessary to sparsify the summed blocks immediately after receiving a sending bag. Reducing unnecessary top-$k$ selection times can reduce the time cost per iteration and accelerate convergence. To solve the SGA dilemma, we only need to ensure that gradients are sparse enough in the next sending bag before the next transmission step. Consequently, we modified the sparsification timing: instead of performing sparsification on blocks after summation, each worker now only sparsifies the blocks destined for the next transmission step after adding the received gradients.




\begin{figure}[t]
\centering
\vspace{-2mm}
\includegraphics[width=0.48\textwidth]{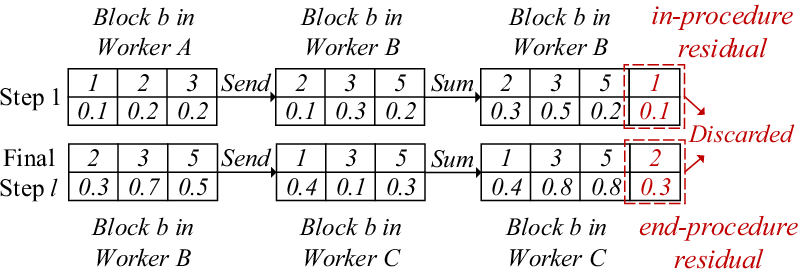}
\vspace{-3mm}
\caption{Illustration of the difference between in-procedure residual and end-procedure residual}
\vspace{-7mm}
\label{fig:res-illustration}
\end{figure}

\vspace{-1mm}
\subsection{Global Residual Collection Algorithm}
\vspace{-1mm}
{\textsf{SparDL} uses multiple top-$k$ selections between transmission steps to solve the SGA dilemma. However, multiple selections discard many gradients, which may contain important gradients. Losing these crucial gradients may lead to slower convergence, e.g., lower accuracy at the same iterations. Therefore, we propose the global residual collection algorithm to store all discarded gradients and ensure fast convergence.}

{Existing compensation methods~\cite{DBLP:conf/ppopp/0002H22, DBLP:conf/icdcs/ShiWZTWHC19, DBLP:conf/iclr/LinHM0D18} collect the discarded gradients from sparsification as residuals and plus them to new gradients at the next iteration. 
There are three types of discarded gradients in \textsf{SparDL}. We call them local residual, end-procedure residual, and in-procedure residual. The local residuals are gradients that are sparsified and discarded at each worker locally before transmission. As for end-procedure residuals and in-procedure residuals, they are both discarded gradients between transmission steps.  
The difference between them is that the gradients corresponding to the end-procedure residuals are completely discarded in the communication process, while the gradients corresponding to the in-procedure residuals are not.
In other words, the indexes of end-procedure residuals do not appear in the final global gradients and the indexes of in-procedure residuals still exist in the final global gradients.
To illustrate the difference between in-procedure residual and end-procedure residual, we consider the example in Figure~\ref{fig:res-illustration}. From this figure, the discarded gradient with index 2 is end-procedure residual and the discarded gradient with index 1 is in-procedure residual. 
The reason is that the gradient with index 2 is completely discarded, and none of the gradient value under this index is shown in the final global gradients; On the contrary, index 1 still exists in the final global gradients, which indicate that there are some gradients with index 1 are selected in the communication process.
The existing compensation methods only collect local residuals~\cite{DBLP:conf/iclr/LinHM0D18} or local and end-procedure residuals~\cite{DBLP:conf/ppopp/0002H22, DBLP:conf/icdcs/ShiWZTWHC19} but cannot collect in-procedure residuals.
Nonetheless, \textsf{SparDL} generates plenty of in-procedure residuals on each worker due to its multiple sparsification processes. As a result, existing methods are unsuitable for \textsf{SparDL}. To address this issue, we propose the global residual collection algorithm, which encompasses the collection of all three types of residuals.}


{
Since the indexes of in-procedure residuals still exist in the final gradients like normal global gradients, it is impossible to collect them only according to the final indexes. Therefore, we collect these residuals throughout the communication process.
As presented in Algorithm~\ref{alg:res}, we collect the residuals in two steps. Since existing methods tend to store the residuals on the worker who generated them, we also store the local and end-procedure residuals on the local worker (lines 13-14). 
But the in-procedure residuals may come from multiple workers, and collecting these residuals on the generated workers requires high communication cost, severely weakening the meaning of sparse communication. For this case, we save this part of residuals directly on the workers who perform the sparsification (line 8). Although these residuals are collected on the processing worker, from the perspective of the entire training cluster, the discarded gradients are still collected and will be used sooner or later.}

\setlength{\textfloatsep}{0pt}
\begin{algorithm}[t]
    \small
    \LinesNumbered
    \caption{\textsf{SparDL} with Global Residual Collection Algorithm}
    \label{alg:res}
    \KwIn{Local gradients $G^{(w, t)}$ at worker $r^w$ and iteration $t$, residuals $\xi^{t-1}_w$ from iteration $t-1$ at worker $r^w$, the number of workers $P$}
    \tcp{\textbf{Worker} $r^w, w \leftarrow 1 ,\ldots, N$ in parallel}
    ${G^{(w, t)}} \leftarrow G^{(w, t)} + \xi^{t-1}_w$.\\
    Partition ${G^{(w, t)}}$ into $P$ local blocks ${G^{(w, t)}} \leftarrow \{{G^{(w, t)}_0} , \ldots, {G^{(w, t)}_P}\}$.\\
    Copy $G^{(w, t)}$ to $G^{(w, t)}_{copy}$.\\
    Partitions blocks of gradients into different bags.
    \For{$step \leftarrow 1$ to $l$}{
        \For{each block $G^{(w, t)}_i$ that will be transmitted next}
    	{
        Sparsify $G^{(w, t)}_i$ by selecting top-$k$ gradients $g^k_i$.\\
        $\xi^{(w, t)}_i \leftarrow G^{(w, t)}_i - {g^k_i}$.\\
        }
        Transmission and summation.\\
    }
    Sparsify the only reserved block $G^{(w, t)}_w$ in worker $r^w$.\\
    All-Gather all blocks from all workers.\\
    \For{${g^k_i} \in \{{g^k_0} , \ldots, {g^k_P}\}$}
	{
        Denote $I_i$ the indexes which are in ${g^k_i}$.\\
        Replace part of $G^{(w, t)}_{copy}$ at indexes $I_i$ with $\xi^{(w, t)}_i$ at indexes $I_i$ .\\
    }
\end{algorithm} 
\setlength{\textfloatsep}{1.55\baselineskip plus 0.2\baselineskip minus 
0.4\baselineskip}

\vspace{-1mm}
\subsection{Spar-All-Gather Algorithm}
\vspace{-1mm}
{The need for low latency and low bandwidth varies in different network environments. Latency cost may dominate the time consumption in some situations. For example, bandwidth cost converges to a constant value when there are plenty of workers in the cluster since the upper bound of bandwidth exists, while latency cost does not converge. Thus, in order to further improve the efficiency of \textsf{SparDL}, it is crucial to reduce the latency cost and make our \textsf{SparDL} adjustable to the focus of the two cost. Therefore, we propose the \textit{Spar-All-Gather} (SAG) algorithm for \textsf{SparDL} based on the idea of divide-and-conquer.}

Specifically, we first divide all $P$ workers into $d$ teams equally, which should satisfy $d\mid P$. Then the $\tfrac{P}{d}$ workers in each team perform the SRS algorithm in parallel. Since there are $\tfrac{P}{d}$ workers in each team, each worker divides their gradients into $\tfrac{P}{d}$ blocks during SRS. After SRS, each worker holds one block of sparse gradients. 
Subsequently, we synchronize different teams and make the workers with the same ranks of each team hold the same $L(k,d,p)=\tfrac{dk}{P}$ sparse gradients after synchronization. We propose two different Spar-All-Gather (SAG) algorithms to synchronize teams in two situations. 
At last, we use Bruck All-Gather in each team.

\vspace{0.025in}
\noindent\textbf{1) $\bm{d}$ is a power of 2.} 
For the case where the number of teams (i.e., $d$) is a power of 2, we propose the recursive-based Spar-All-Gather (R-SAG). R-SAG is based on recursive doubling~\cite{DBLP:journals/ijhpca/ThakurRG05}, as shown in Figure~\ref{fig:reduceallgather}(a), and combines top-$k$ selection after each transmission step. 
Specifically, the R-SAG's communication occurs between teams and teams, with the $t$-th transmission step when the team $X$ and a team $Y$ at a distance of $2^t$ exchange all data. For a worker in $X$, it communicates with the worker at the same position in $Y$. They exchange the only block of sparse gradients they hold. Unlike SRS, the sent block is still held after transmission. Then, each worker adds the received gradients to the held gradients. Due to index inconsistency, the summed sparse gradients are also subject to SGA problems. Thus, we sparsify the gradients by selecting the top-$L(k,d,p)$ values. 

In recursive doubling communication, the source and target in one transmission are the same for any worker. Thus, both sides of the transmission hold the same sparse gradients after summation and discard the same gradients after selection. Therefore, we collect half value of the discarded gradients as residuals on each side of the transmission.

After communicating $\log_2(d)$ times, all teams will be synchronized. Then we All-Gather sparse gradients of each worker using Bruck All-Gather in each team. 

\vspace{0.025in}
\noindent\textbf{Communication Complexity Analysis.} 
Since the size of the message in each transmission is $2L(k,d,p)$ (indexes and values), the cost of R-SAG step is 
\vspace{-1mm}
\begin{equation}
\vspace{-1mm}
T_{2}' = \log_2d\alpha + 2\tfrac{dk}{P}\log_2d\beta.
\end{equation}
We still use Bruck All-Gather. Hence, the communication cost of SRS and the final All-Gather is both 
\vspace{-1mm}
\begin{equation}
\vspace{-1mm}
T_{1}' = T_{3}' = \lceil \log_2\tfrac{P}{d} \rceil\alpha + 2k\tfrac{P-d}{P}\beta. 
\end{equation}
Therefore, the total communication cost of \textsf{SparDL} with R-SAG is:
\vspace{-1mm}
\begin{equation}
\begin{aligned}
T_{all}' = T_{1}'+T_{2}'+T_{3}' = (2\lceil \log_2\tfrac{P}{d} \rceil + \log_2d)\alpha  \\ + 2k(\tfrac{2P-2d}{P}+\tfrac{d}{P}\log_2d)\beta.
\end{aligned}
\end{equation}

\vspace{0.025in}
\noindent\textbf{{The impact of $d$ in R-SAG}.} 
{
Since $d$ in R-SAG is a power of 2, we analyze the change when increasing $d$ to $2d$. Then, we can find that latency cost decreases $\alpha$, but bandwidth cost increases by $2\frac{dk}{p}\log_2d\beta$.
Thus, when $d$ is set to 2, \textsf{SparDL} with R-SAG has lower latency cost than \textsf{SparDL} without SAG (i.e., $d=1$) and has the same bandwidth cost. 
Then, with the increase of $d$, the increase in bandwidth cost becomes progressively more pronounced, possibly exceeding the reduction in latency cost when $d$ is too large.
}

\vspace{0.025in}
\noindent\textbf{2) $\bm{d}$ is not a power of 2.} 
For the case where the number of teams is not a power of 2, recursive doubling cannot be used directly. Instead, the Bruck All-Gather algorithm can be utilized to synchronize gradients among different teams.
However, the communication consistency will be damaged when Bruck All-Gather combines with sparsification like Spar-Reduce-Scatter and R-SAG to solve the SGA issue.
As shown in Figure~\ref{fig:reduceallgather}(b), if we use top-$L(k,d,p)$ selection at each step after addition, then at the end of the communication, the gradient blocks held by worker 1 and worker 2 experience entirely different compression orders, e.g., block 1 is compressed with block 6 in worker 1 but with block 2 in worker 2 after step 1. Thus, the final sparse gradients are different for worker 1 and worker 2.
Such inconsistency defeats the purpose of synchronous SGD in training distributed deep learning models and will make different workers hold different deep learning models after several iterations. 
Thus, the SGA problem still exists when using the Bruck All-Gather to synchronize gradients among different teams.

\begin{figure}[tb]
	\centering
	\vspace{-1mm}
	\includegraphics[width=0.42\textwidth]{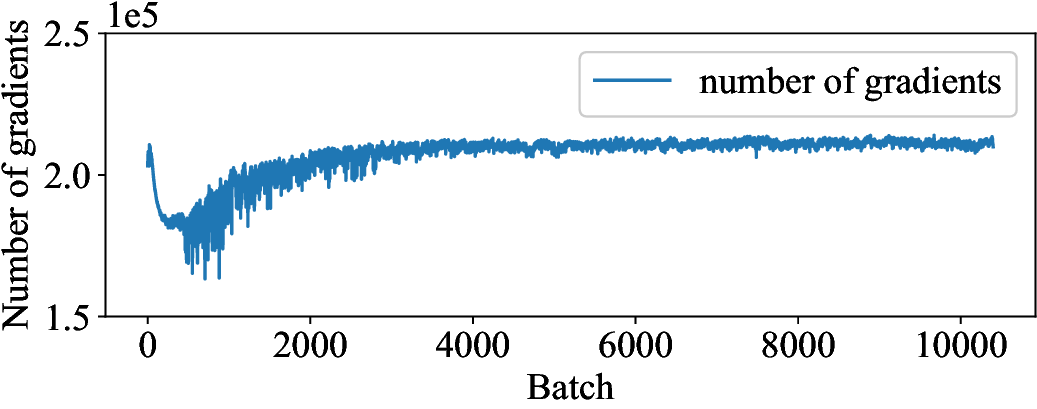}
	\vspace{-4mm}
	\caption{Number of sparse gradients after using Bruck All-Gather to synchronize gradients among teams during 20 epochs}
	\vspace{-5mm}
	\label{fig:bsag}
\end{figure}

\vspace{0.025in}
\noindent\textbf{Observation.} 
Previous study~\cite{DBLP:conf/ppopp/0002H22} has observed that the indexes distribution of the selected sparse gradients from each worker changes slowly with regard to iterations. Besides, the number of sparse gradients, after using Bruck All-Gather to synchronize gradients among teams, also changes slowly, as shown in Figure~\ref{fig:bsag}.
Thus, if we make an additional top-$h$ selection with a suitable number $h$ before Bruck All-Gather communication and change $h$ slowly in the subsequent communication rounds, we can make the number of gradients after Bruck All-Gather near $L(k,d,p)$. In this way, using additional top-$h$ selection can reduce the bandwidth cost with no excessive gradients lost compared to using selection during communication. Therefore, we propose the Bruck-based sparse All-Gather algorithm (B-SAG).

\vspace{0.025in}
\noindent\textbf{Solution.} 
First, we specify the range of $[\tfrac{k}{P},\tfrac{dk}{P}]$ for $h$ when the number of gradients after B-SAG is equal to $L(k,d,p)$. The lower and upper bounds are taken as entirely non-overlapping and entirely overlapping gradient indexes between different workers, respectively. Second, through the Figrue~\ref{fig:bsag}, we can find that the number of gradients after Bruck All-Gather is stable within successive iterations. And there are some fluctuations throughout the training process. 
Thus, we need to obtain a suitable value of $h$ as soon as possible to make the number of gradients after B-SAG close to $L(k,d,p)$, and then adjust $h$ according to the fluctuation.
Therefore, we propose a compression ratio $h$ adjustment algorithm for B-SAG, which is motivated by the CWnd algorithm~\cite{DBLP:journals/ccr/DukkipatiRCCHAJS10}.

As shown in Algorithm~\ref{alg:nspardl},
we adjust $h$ by a step size with direction (i.e., sign). The initial $h$ is $\tfrac{k}{P}$, and the initial step size is $0.01\times k\tfrac{d-1}{P}$ with positive direction. After each iteration, we set $h=h+step$ and adjust the step according to the number of gradients amount after Bruck All-Gather in this iteration~(line 12). If the gradient amount is greater than $L(k,d,p)$, we set the direction as positive, and if the gradient amount is not greater than $L(k,d,p)$, we set the direction as negative. In addition, if two consecutive step adjustments are in the same direction, we double the step~(lines 3-8), and if the direction is changed, we divide the step by 2~(lines 9-11). Through the step adjustments, we can change $h$ to make the number of gradients recover to near $L(k,d,p)$ as soon as possible and adjust $h$ slowly when the number changes slowly.

After completing the sparse Bruck All-Gather, 
we sparsified the number of sparse gradients to $L(k,d,p)$. 
After B-SAG, the amount of gradients held by each worker may be greater than $L(k,d,p)$, so we perform sparsification on each worker. Since each team has the same gradients, each worker collects $\tfrac{1}{d}$ value of the discarded gradients as residuals. 
Besides, since different workers in the same team still have blocks of gradients with different positions, we All-Gather these blocks by Bruck All-Gather as we do in \textsf{SparDL} with R-SAG.

\setlength{\textfloatsep}{0pt}
\begin{algorithm}[t]
    \small
    \LinesNumbered
    \caption{Compression Ratio Adjustment Algorithm for B-SAG}  
    \label{alg:nspardl}
    \KwIn{The number of workers $P$, the number of teams $d$, the number of gradients after B-SAG $N_t$ at iteration $t$, the selecting number $k$}
    Initial $h \leftarrow \tfrac{k}{P}$, the step size $step \leftarrow 0.01\times k\tfrac{d-1}{P}$, $flag \leftarrow False$, and $L(k,d,p) \leftarrow \tfrac{dk}{P}$.\\
    \For{$t \leftarrow 1$ to $T$}{
        \eIf{$N_t > L(k,d,p) \oplus step > 0$}{
            \eIf{$flag$}{
                $step \leftarrow step \times 2$.\\
                $flag \leftarrow False$.\\
            }{
                $flag \leftarrow True$.\\
            }
        }{
                $step \leftarrow -step \times \tfrac{1}{2}$.\\
            $flag \leftarrow False$.\\
        }
        $h \leftarrow h+step$.\\
        Proceed B-SAG with top-$h$ sparsification.\\
        Count $N_{t}$ as the number of gradients after B-SAG.\\
    }
\end{algorithm}
\setlength{\textfloatsep}{1.55\baselineskip plus 0.2\baselineskip minus 
0.4\baselineskip}

\vspace{0.025in}
\noindent\textbf{Communication Complexity Analysis.} 
The size of gradients sending before B-SAG is $2h\in[2\tfrac{k}{P},2\tfrac{dk}{P}]$. Thus, the bandwidth overhead of B-SAG is $[2k\tfrac{d-1}{P}\beta, 2k\tfrac{d^2-d}{P}\beta]$. The communication cost of B-SAG is:
\vspace{-2mm}
\begin{equation}
\vspace{-2mm}
\lceil \log_2(d)\rceil\alpha  + 2k\tfrac{d-1}{P}\beta \leq T_{2}' \leq \lceil \log_2(d)\rceil\alpha  + 2k\tfrac{d^2-d}{P}\beta.
\end{equation}
The size of gradients after B-SAG is $[2\tfrac{k}{P},2\tfrac{dk}{P}]$, the communication cost of final All-Gather is 
\vspace{-2mm}
\begin{equation}
\vspace{-2mm}
\lceil \log_2\tfrac{P}{d} \rceil\alpha + 2k\tfrac{P-d}{Pd}\beta \leq T_{3}' \leq \lceil \log_2\tfrac{P}{d} \rceil\alpha + 2k\tfrac{P-d}{P}\beta. 
\end{equation}
Therefore, the total communication overhead of \textsf{SparDL} with B-SAG is:
\vspace{-2mm}
\begin{equation}
\begin{aligned}
(2\left\lceil \log_2\tfrac{P}{d} \right\rceil+\lceil\log_2 d\rceil)\alpha + 2k\tfrac{d^2+P-2d}{Pd}\beta \\ \leq T_{all}' = T_{1}'+T_{2}'+T_{3}' \\
\leq (2\left\lceil \log_2\tfrac{P}{d} \right\rceil+\lceil\log_2 d\rceil)\alpha + 2k\tfrac{d^2+2P-3d}{P}\beta.
\end{aligned}
\end{equation}


\vspace{0.025in}
\noindent\textbf{{The impact of $d$ in B-SAG}.} 
{
$d$ of B-SAG may not be a power of 2, but, like R-SAG, every time $d$ is increased to $2d$, latency cost will be reduced by $\alpha$.
Different from R-SAG, the bandwidth cost of B-SAG has a value range, so we consider it from the lower bound and the upper bound, respectively. The lower bound of bandwidth cost decreases with the increase of $d$ when $d < \sqrt{p}$, but increases when $d > \sqrt{p}$; the upper bound of bandwidth cost increases when $d > 1.5$, and the upper bound of bandwidth cost when $d=2$ is the same as the bandwidth cost of \textsf{SparDL} without SAG (i.e., $d=1$). 
Thus, an appropriate $d$ reduces the communication time of \textsf{SparDL} with B-SAG, but a too large $d$ may cause an increase in the communication time since the increase in bandwidth cost exceeds the decrease in latency cost.
}

\vspace{0.025in}
\noindent\textbf{{The selection of the optimal $d$ for SAG algorithm.}}
{For each network condition and model size, there is an optimal number of groups $d$ to make R-SAG or B-SAG get the fastest training efficiency. Through the experiment (in Section~\ref{sec: adjustd}), we find that the time consumption of each epoch is relatively stable. Thus, the $d$ with the least time consumption in the first epoch is likely to be the optimal $d$. 
Besides, since $d$ should satisfy $d\mid P$, there are only a few available values for $d$.
Therefore, to select the optimal number of groups $d$, we suggest that users run one epoch for each $d$, and choose $d$ with the least time consumption as the optimal $d$.}

\noindent\textbf{{Computation and Space Complexity Analysis for \textsf{SparDL}.}} 
{
Let $P$ be the number of workers, $n$ be the number of dense gradients, $k$ be the number of sparse gradients ($k \ll n$), $d$ be the number of teams in \textsf{SparDL} with SAG. 
\textsf{SparDL} first partitions dense gradients into $\frac{P}{d}$ blocks and $l = \lceil \log_2\frac{P}{d} \rceil$ sending bags, which takes $O(\frac{P}{d})$ time and $O(1)$ space. 
Then, \textsf{SparDL} directly restores all gradients for the global residual collection algorithm and takes $O(n)$ time and $O(n)$ space. 
Next, \textsf{SparDL} transmits gradients $l$ times, which takes $O(\frac{P}{d}\frac{kd}{P})=O(k)$ time and $O(n)$ space for gradients summation and receiving buffer, and sparsifies $\frac{P}{d}$ blocks of gradients to select $\frac{k}{\frac{P}{d}}=\frac{kd}{P}$ gradients each block, which takes $O(\frac{P}{d}\frac{nd}{P})=O(n)$ time and $O(\log \frac{nd}{P})$ space to select top-$\frac{kd}{P}$ gradients using Quicksort-based algorithms~\cite{DBLP:conf/icpp/XueLZGX16, DBLP:conf/iclr/LinHM0D18}, in the Spar-Reduce-Scatter step.}

{
Then, if $d>1$, we use R-SAG or B-SAG. Otherwise, we skip to the next step.
For R-SAG, \textsf{SparDL} transmits and sparsifies gradients $\log_2 d$ times, which takes $O(\frac{kd}{P}\log d)$ time and $O(\log \frac{kd}{P})$ space for sparsification and $O(\frac{kd}{P}\log d)$ time and $O(\frac{nd}{P})$ space for transmission.}
{For B-SAG, \textsf{SparDL} sparsifies gradients twice and transmits gradients $\lceil \log_2 d \rceil$ times. This step takes $O(\frac{kd}{P})$ time and $O(\log \frac{kd}{P})$ space for sparsification and $O(\frac{kd^2}{P})$ time and $O(n)$ space for transmission.
}

{
After SAG step, \textsf{SparDL} All-Gathers all blocks and takes $O(1)$ time and $O(n)$ space. 
At last, the global residual collection algorithm takes $O(n)$ time and $O(1)$ space at most.
Overall, \textsf{SparDL} (R-SAG) takes $O(n+\frac{kd}{P}\log d)$ computation complexity; \textsf{SparDL} (B-SAG) takes $O(n+\frac{kd^2}{P})$ computation complexity; \textsf{SparDL} without SAG takes $O(n)$ computation complexity. They all take $O(n)$ space complexity.
}
\section{Experiments}
\label{sec:exe}
\subsection{Experimental Setup}
\vspace{-1mm}
In this section, we evaluate the effectiveness and efficiency of the proposed \textsf{SparDL} using multiple practical deep learning models and datasets, compared with four baselines.

\begin{table}[t]
\caption{Deep learning cases used for evaluation}
\small
\vspace{-3mm}
\centering
\begin{tabular}{c|c|l|l}
\hline
Case  & Task Type            & Models    & Dataset   \\ \hline
Case 1 & Type 1 & VGG-16~\cite{DBLP:journals/corr/SimonyanZ14a}    & CIFAR-10~\cite{krizhevsky2009learning}  \\
Case 2 & Type 1 & VGG-19    & CIFAR-100 \\
Case 3 & Type 1 & ResNet-50~\cite{DBLP:conf/cvpr/HeZRS16} & ImageNet~\cite{DBLP:journals/ijcv/RussakovskyDSKS15}  \\
Case 4 & Type 2     & VGG-11    & House~\cite{DBLP:conf/ijcci/AhmedM16}     \\
Case 5 & Type 3  & LSTM-IMDB   & IMDB~\cite{maas-EtAl:2011:ACL-HLT2011}      \\
Case 6 & Type 4  & LSTM-PTB~\cite{hochreiter1997long}  & PTB~\cite{DBLP:journals/coling/MarcusSM94}       \\
Case 7 & Type 5  & BERT~\cite{DBLP:conf/naacl/DevlinCLT19}      & Wikipedia~\cite{DBLP:conf/naacl/DevlinCLT19}      \\ \hline
\end{tabular}
\label{tab:models}
\vspace{-2mm}
\end{table}

\vspace{0.025in}
\noindent\textbf{Datasets and models.} 
To comprehensively evaluate the efficacy of the proposed \textsf{SparDL} in diverse scenarios, we leverage five distinct types of tasks: image classification (Type 1), image regression (Type 2), text classification (Type 3), language modeling (Type 4), and language processing (Type 5). We use seven different real-world datasets and seven widely-utilized deep learning models.
Specfically, these models contain 14.7M, 20.1M, 23.5M, 9.2M, 35.2M, 66M, 133.5M parameters, respectively. LSTM-IMDB and LSTM-PTB are 2-layer RNN models with LSTM units, and LSTM-PTB is similar as in \cite{DBLP:conf/iclr/LinHM0D18} and~\cite{DBLP:conf/icdcs/ShiWZTWHC19}. 
The deep learning models and datasets are summarized in Table~\ref{tab:models}.

\vspace{0.025in}
\noindent\textbf{Competitors.} 
We compare the proposed \textsf{SparDL} with four existing sparse All-Reduce methods: gTop$k$~\cite{DBLP:conf/icdcs/ShiWZTWHC19, DBLP:conf/ijcai/ShiZWTC19}, O$k$-Top$k$~\cite{DBLP:conf/ppopp/0002H22}, Top$k$A~\cite{DBLP:conf/sc/RenggliAAAH19}, and Top$k$DSA~\cite{DBLP:conf/sc/RenggliAAAH19}. 

\vspace{0.025in}
\noindent\textbf{Experimental setting.} All experiments in this paper are implemented in PyTorch 1.11.0~\cite{DBLP:conf/nips/PaszkeGMLBCKLGA19} with mpi4py 3.0.3.
All experiments are evaluated in a GPU cluster with 14 GPU machines. Each GPU machine is installed with CentOS-7.9, Nvidia driver 510.39.01, and CUDA 11.6 and has 256 GB RAM, two Intel(R) Xeon(R) 4314 CPU @ 2.40GHz processors, one GeForce RTX A40 GPU.
All machines are connected to an Ethernet with default setting.

\begin{figure}[t]
\centering
\hspace{-3mm}
\includegraphics[width=0.33\textwidth]{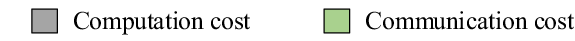}

\centering
\vspace{-2mm}
\hspace{-3mm}
\subfigure[VGG-19 on CIFAR100]{
\includegraphics[width=0.21\textwidth]{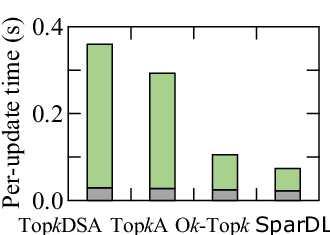}
}
\hspace{1.4mm}
\hspace{-0.5cm}
\subfigure[VGG-11 on House]{
\includegraphics[width=0.21\textwidth]{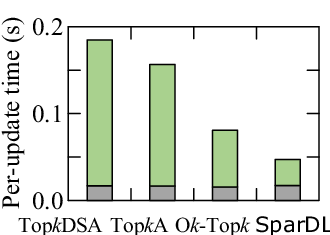}
}
\vspace{-3mm}

\centering
\hspace{-3mm}
\subfigure[LSTM-IMDB on IMDB]{
\includegraphics[width=0.21\textwidth]{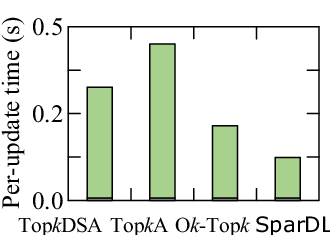}
}
\hspace{1.4mm}
\hspace{-0.5cm}
\subfigure[LSTM-PTB on PTB]{
\includegraphics[width=0.21\textwidth]{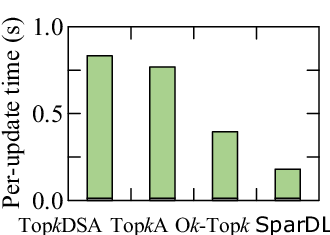}
}
\vspace{-2mm}

\caption{Per-update time with 14 workers}
\vspace{-6mm}
\label{fig:iteration_time}
\end{figure}

\vspace{-2mm}
\subsection{Performance in Four Deep Learning Cases}
\vspace{-1mm}
To verify the superiority of our proposed \textsf{SparDL} framework on communication cost, we compare it with O$k$-Top$k$~\cite{DBLP:conf/ppopp/0002H22}, Top$k$A~\cite{DBLP:conf/sc/RenggliAAAH19} and Top$k$DSA~\cite{DBLP:conf/sc/RenggliAAAH19} on four distinct distributed deep learning cases where 14 workers involved. These four cases are corresponding to image classification, image regression, text classification, and language modeling, respectively. The experimental results are illustrated in Fig.\ref{fig:iteration_time}. From this figure, we can observe that \textsf{SparDL} has the lowest communication cost in all four cases. Specifically, Fig.\ref{fig:iteration_time}(a) shows that, when training VGG-19 on CIFAR-100, \textsf{SparDL} exhibits $6.4\times$ faster than Top$k$DSA, $5.1\times$ faster than Top$k$A, and $1.6\times$ faster than O$k$-Top$k$ for communication cost individually. Similarly, in Fig.~\ref{fig:iteration_time}(b), while training VGG-11 on House, \textsf{SparDL} is the fastest and achieves $5.6\times, 4.7\times, 2.2\times$ speedup over the state-of-the-art methods in communication cost. From Fig.~\ref{fig:iteration_time}(c) and Fig.~\ref{fig:iteration_time}(d), \textsf{SparDL} is $2.7\times, 3.8\times, 1.8\times$ and $5.0\times, 4.5\times, 2.3\times$ faster than the baselines on LSTM-IMDB and LSTM-PTB, respectively.

Among the compared methods, Top$k$DSA is the slowest due to its incomplete resolution of the Sparse Gradient Accumulation (SGA) dilemma, leading to increasing transmission volume after each step in one iteration. Besides, the performance of Top$k$A is also inferior, for it solves the SGA dilemma only using the All-Gather operation, incurring high bandwidth cost even with the fast recursive doubling algorithm. Additionally, \textsf{SparDL} also significantly outperforms O$k$-Top$k$ in communication cost, since O$k$-Top$k$ introduces many additional transmission operations to balance gradients among workers for solving the SGA dilemma, resulting in high latency cost and large upper bound of bandwidth cost. Moreover, O$k$-Top$k$ employs threshold pruning instead of top-$k$ selection, which leads to the actual time consumption of O$k$-Top$k$ exceeding its theoretical time consumption. In contrast, \textsf{SparDL}, utilizing top-$k$ pruning, avoids this issue entirely. 
We can also observe that the communication cost of training VGG-11 is lower than that of VGG-19 for each method. The reason is that VGG-19 has more parameters than VGG-11, and more parameters cause more bandwidth cost, subsequently increasing the communication time. For the same reason, the communication cost of training LSTM-PTB is higher than LSTM-IMDB. As a result, \textsf{SparDL} consistently achieves faster iteration speeds than Top$k$DSA, Top$k$A, and O$k$-Top$k$ across various task types, models, and datasets.

\begin{figure}[t]
\centering
\hspace{-3mm}
\includegraphics[width=0.5\textwidth]{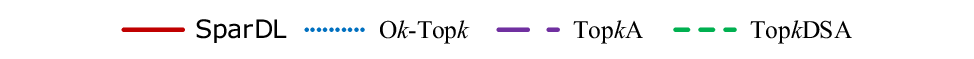}
\vspace{-5mm}

\centering
\hspace{-3mm}
\subfigure[VGG-19 on CIFAR-100]{
\includegraphics[width=0.21\textwidth]{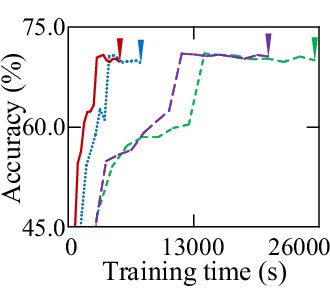}
}
\hspace{1.4mm}
\hspace{-0.5cm}
\subfigure[VGG-11 on House]{
\includegraphics[width=0.21\textwidth]{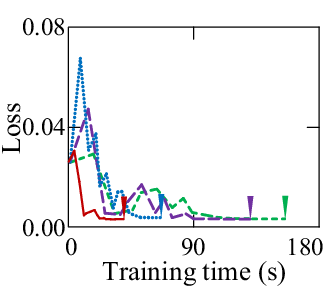}
}
\vspace{-3mm}

\centering
\hspace{-3mm}
\subfigure[LSTM-IMDB on IMDB]{
\includegraphics[width=0.21\textwidth]{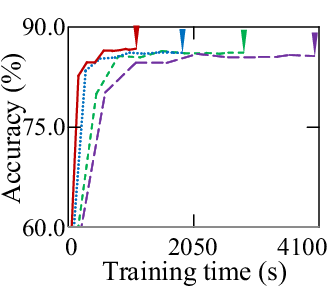}
}
\hspace{1.4mm}
\hspace{-0.5cm}
\subfigure[LSTM-PTB on PTB]{
\includegraphics[width=0.21\textwidth]{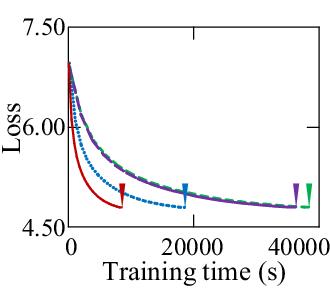}
}
\vspace{-2mm}

\caption{Training different distributed deep learning cases with 14 workers}
\vspace{-6mm}
\label{fig:convergence}
\end{figure}

\vspace{-2mm}
\subsection{Convergence in Four Deep Learning Cases}
\vspace{-1mm}
To demonstrate \textsf{SparDL}'s capability in reducing the overall convergence time, we record the test accuracy (or test loss) of \textsf{SparDL} w.r.t. the training time and compare it with that for the baselines, illustrated in Fig.~\ref{fig:convergence}. 
From the experimental results, it is evident that \textsf{SparDL} outperforms all the baseline methods, exhibiting the shortest training completion time while converging to similar accuracy (or loss) as the baseline methods. 
Specifically, as shown in Fig.~\ref{fig:convergence}, \textsf{SparDL} achieves $4.9\times, 4.0\times, 1.4\times$ faster than Top$k$A, Top$k$DSA and O$k$-Top$k$ for VGG-19, respectively. In the case of VGG-11, the speedup is $3.9\times, 3.3\times, 1.7\times$. As for LSTM-IMDB, the speedup is $2.6\times, 3.6\times, 1.7\times$ over the baselines. And for LSTM-PTB, \textsf{SparDL} is $4.6\times, 4.3\times$ and $2.2\times$ faster, respectively. Moreover, from Fig.~\ref{fig:convergence}(a) to Fig.~\ref{fig:convergence}(d), the models trained with the four sparse All-Reduce methods all converge to similar accuracy (or loss) after an equivalent number of epochs. 

\textsf{SparDL} expedites the training process by accelerating the communication process within each iteration. Since the computation cost is stable using different communication methods, the acceleration impact of \textsf{SparDL} on per-update time is marginally inferior to its impact on communication cost. Nonetheless, \textsf{SparDL} still achieves a considerable acceleration in the training process. This is attributed to its efficient resolution of the SGA dilemma without requiring additional transmission. It accomplishes this by partitioning gradients into blocks and maintaining the block size to ensure optimal efficiency, which substantially improves communication speed. Besides, \textsf{SparDL} can achieve comparable accuracy or loss as the baselines after the same number of epochs. This is credited to the proposed global residual collection algorithm employed by \textsf{SparDL}, enabling it to collect all gradients pruned in each top-$k$ selection for re-utilization and thereby maintaining \textsf{SparDL}'s convergence rate. In summary, \textsf{SparDL} has comparable convergence rate as Top$k$A, Top$k$DSA and O$k$-Top$k$ while accelerating training across different tasks with different model types and data types.

\begin{figure}[t]
\centering
\hspace{-3mm}
\includegraphics[width=0.33\textwidth]{fig-icde/ICDE-E01-0.eps}

\centering
\vspace{-2mm}
\hspace{-3mm}
\subfigure[ResNet-50 on ImageNet]{
\includegraphics[width=0.21\textwidth]{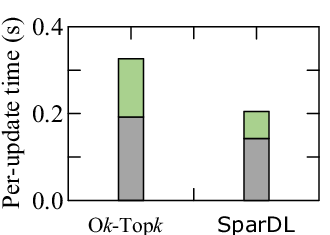}
}
\hspace{1.4mm}
\hspace{-0.5cm}
\subfigure[BERT on Wikipedia]{
\includegraphics[width=0.21\textwidth]{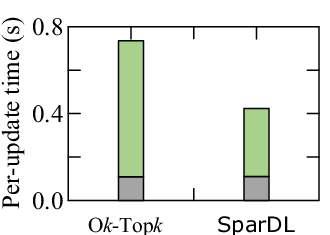}
}
\vspace{-2mm}
\caption{Per-update time on ResNet-50 and BERT with 14 workers}
\vspace{-4mm}
\label{fig:berttime}
\end{figure}

\begin{figure}[t]
\centering
\hspace{-3mm}
\includegraphics[width=0.5\textwidth]{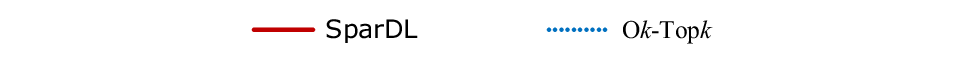}
\vspace{-5mm}

\centering
\vspace{-3mm}
\hspace{-3mm}
\subfigure[ResNet-50 on ImageNet]{
\includegraphics[width=0.21\textwidth]{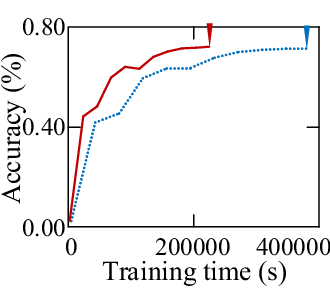}
}
\hspace{1.4mm}
\hspace{-0.5cm}
\subfigure[BERT on Wikipedia]{
\includegraphics[width=0.21\textwidth]{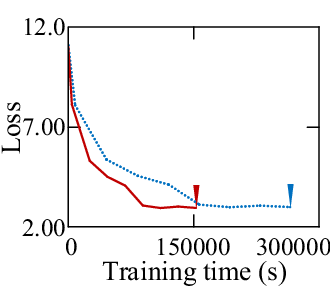}
}
\vspace{-2mm}

\caption{Convergence on ResNet-50 and BERT with 14 workers}
\label{fig:bertloss}
\vspace{-6mm}
\end{figure}


\vspace{-2mm}
\subsection{Comparison on large datasets: ImageNet and Wikipedia with ResNet-50 and BERT}
To further validate the efficiency of \textsf{SparDL}, this paper further evaluates \textsf{SparDL} with the large datasets  (i.e., ImageNet and Wikipedia) and large models (i.e., ResNet-50 and BERT). In these experiments, we compare \textsf{SparDL} with O$k$-Top$k$, which is the most efficient method among the baselines. The results are illustrated in Fig.~\ref{fig:berttime}, which verifies the superiority of our proposed \textsf{SparDL}. Fig.~\ref{fig:berttime}(a) and Fig.~\ref{fig:berttime}(b) depict the per-update time. For ResNet-50, we can observe that \textsf{SparDL} achieves $2.3\times$ acceleration of communication cost compared to O$k$-Top$k$. For BERT, \textsf{SparDL} is $2.0\times$ faster than O$k$-Top$k$ in terms of communication cost. 
Fig.~\ref{fig:bertloss}(a) and Fig.~\ref{fig:bertloss}(b) plot the test accuracy and training loss w.r.t. the training time of ResNet-50 and BERT, respectively. These figures show that \textsf{SparDL} maintains convergence rate comparable to O$k$-Top$k$ while completing the training process obviously faster than O$k$-Top$k$ on large datasets. 
Besides, \textsf{SparDL} achieves $1.7\times$ speedup over O$k$-Top$k$ for ResNet-50, and a similar speedup of $1.7\times$ for BERT. Therefore, \textsf{SparDL} demonstrates efficiency and effectiveness 
on large datasets with different task types.

\begin{figure}[t]
\centering
\hspace{-3mm}
\includegraphics[width=0.5\textwidth]{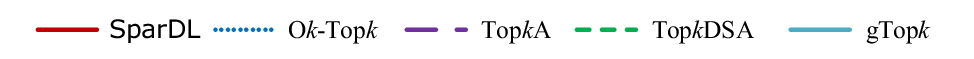}
\vspace{-5mm}

\centering
\vspace{-2mm}
\hspace{-3mm}
\subfigure[Speedup w.r.t. number of workers]{
\includegraphics[width=0.22\textwidth]{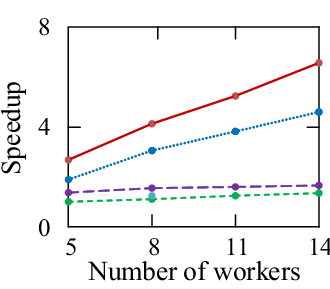}
}
\hspace{1.4mm}
\subfigure[Case 2 with 8 workers]{
\includegraphics[width=0.22\textwidth]{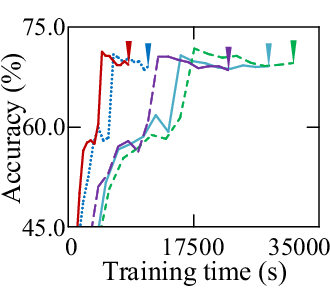}
}

\caption{Performance and convergence with different number of workers}
\vspace{-6mm}
\label{fig:scalability}
\end{figure}

\vspace{-2mm}
\subsection{Scalability}
\label{Scalability}
\vspace{-1mm}
In this set of experiments, we evaluate the scalability of \textsf{SparDL} with evaluation metrics delineated in literature~\cite{DBLP:conf/icde/ZhouLLOWY21} and compare it with that for the baselines. The experimental results are recorded in Fig.~\ref{fig:scalability}, which demonstrates that \textsf{SparDL} has the highest scalability. Specifically, we set the average training time required to complete a single epoch of VGG-19 on CIFAR-100, employing Top$k$DSA as the communication method within an 8-worker cluster, as the reference time. Subsequently, we compute the speedup achieved by gTop$k$, Top$k$A, Top$k$DSA, and O$k$-Top$k$ with varying numbers of workers, relative to this reference time. We only evaluate gTop$k$ with 8 workers since it only works in clusters with power-of-two workers. As shown in Fig.~\ref{fig:scalability}(a), \textsf{SparDL} exhibits superior scalability compared to other sparse All-Reduce methods. Besides, Fig.~\ref{fig:scalability}(b) depicts the test accuracy w.r.t. training time of VGG-19 using gTop$k$, Top$k$A, Top$k$DSA, O$k$-Top$k$ and \textsf{SparDL} with 8 workers. It is evident that \textsf{SparDL} surpasses others in terms of speed.

The superior scalability of \textsf{SparDL} can be attributed to \textsf{SparDL}'s lower communication complexity relative to the other approaches. As the number of workers $P$ increases, the speed gap between \textsf{SparDL} and the other methods widens. 
Besides, from Fig.~\ref{fig:scalability}(b), we can observe that the acceleration margin of \textsf{SparDL} over other methods with 8 workers is less pronounced than when trained with 14 workers, given \textsf{SparDL}'s enhanced speedup with more workers. 
Additionally, gTop$k$ is slower than \textsf{SparDL}, mainly attributable to its inefficient synchronization of gradients via reduction tree and broadcast tree models, resulting in high bandwidth cost.


\begin{figure}[t]
\centering
\vspace{-2mm}
\hspace{-3mm}
\subfigure[\textsf{SparDL} with R-SAG]{
\includegraphics[width=0.21\textwidth]{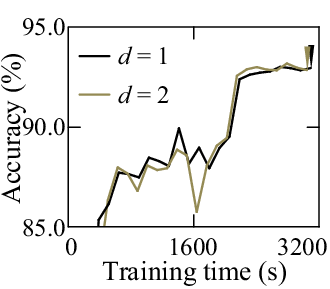}
}
\hspace{1.4mm}
\hspace{-0.5cm}
\subfigure[\textsf{SparDL} with B-SAG]{
\includegraphics[width=0.21\textwidth]{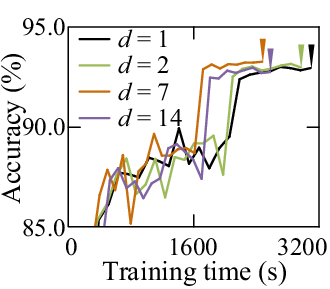}
}
\vspace{-2mm}
\caption{\textsf{SparDL} with different SAG algorithms in the cluster of 14 workers}
\vspace{-2mm}
\label{fig:sag}
\end{figure}

\begin{figure}[t]
\centering
\vspace{-2mm}
\hspace{-3mm}
\subfigure[Imapact of $d$ with 14 workers]{
\includegraphics[width=0.21\textwidth]{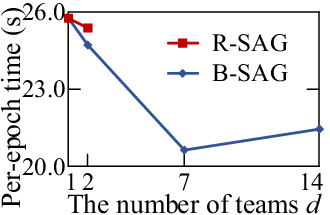}
}
\hspace{1.4mm}
\hspace{-0.5cm}
\subfigure[Imapact of $d$ with 12 workers]{
\includegraphics[width=0.21\textwidth]{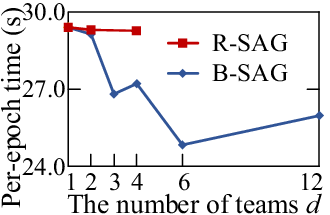}
}
\vspace{-2mm}
\caption{{Impact of $d$ with different number of workers}}
\vspace{-6mm}
\label{fig:sag-new}
\end{figure}

\vspace{-2mm}
\subsection{Impact of Spar-All-Gather algorithm.}
\vspace{-1mm}
In this set of experiments, we evaluate the impact of the Spar-All-Gather (SAG) algorithm on \textsf{SparDL}. The SAG algorithm is proposed to reduce the latency cost and further improve efficiency. And there are two SAG algorithms, i.e., R-SAG and B-SAG. The experimental results are depicted in Fig.~\ref{fig:sag}
, which prove that both SAG algorithms can accelerate \textsf{SparDL}.
Specifically, we set the team number, $d$, to 1 and 2 for R-SAG, and set $d$ to 1, 2, 7, and 14 for B-SAG. Afterward, the effectiveness of these algorithms is evaluated via the training of VGG-16 on CIFAR-10, utilizing \textsf{SparDL} with either R-SAG or B-SAG as the communication method. 
As shown in Fig.~\ref{fig:sag}(a), \textsf{SparDL} with R-SAG ($d=2$) complete the training process slightly faster compared with \textsf{SparDL} without SAG, i.e., \textsf{SparDL} with R-SAG or B-SAG ($d=1$). Besides, from Fig.~\ref{fig:sag}(b), all \textsf{SparDL} frameworks with B-SAG ($d>1$) are faster than \textsf{SparDL} without SAG, signifying that B-SAG accelerates \textsf{SparDL}. 
In particular, B-SAG ($d=7$) and B-SAG ($d=14$) can significantly improve the speed by $ 1.25\times$ and $1.2\times$, respectively. 
Additionally, the experimental results also show that \textsf{SparDL} frameworks with SAG have similar convergence rate as \textsf{SparDL} without SAG.

{
Next, we observe the influence of different $d$ on the efficiency of SAG with different numbers of workers. 
For 14 workers in Fig.~\ref{fig:sag-new}(a), the reason why R-SAG has this effect on \textsf{SparDL} is that R-SAG with 2 teams marginally decreases the latency cost of \textsf{SparDL} while keeping the bandwidth cost unchanged, thereby causing a slight reduction in training time. Besides, B-SAG greatly improves the speed since it reduces not only the latency cost but also the bandwidth cost of \textsf{SparDL}.
We can also observe that B-SAG ($d=14$) is slower than B-SAG ($d=7$). This is because the upper bound of bandwidth increases with the increase of $d$, and the lower bound of bandwidth also increases with the increase of $d$ when $d>\sqrt{p}$. Thus, a large $d$ eventually weakens the effect of improving the speed compared with the best $d$ ($d=7$).
For 12 workers in Fig.~\ref{fig:sag-new}(b), we can observe that the efficiency improvement of R-SAG from 2 to 4 is less than that from 1 to 2. The reason is that R-SAG ($d=2$) decreases the latency cost of \textsf{SparDL} while keeping the bandwidth cost unchanged, but the R-SAG ($d=4$) reduces the same latency cost while increasing the bandwidth cost. As for B-SAG, B-SAG ($d=4$) is slower than B-SAG ($d=3$). It is because B-SAG ($d=4$) has the same latency cost and the same lower bound of bandwidth cost as B-SAG ($d=3$), and the upper bound of bandwidth cost of B-SAG ($d=4$) is higher.}
In addition, \textsf{SparDL} with B-SAG ($d=14$) attains the lowest test accuracy, as shown in Fig.~\ref{fig:sag}(b).
This arises when $d$ equals the number of workers (i.e., $P$) in the cluster, creating a scenario where all workers aggregate after just one local top-$h$, without consideration of global gradients at all. This leads to the discarding of many globally significant gradients but locally less important, thus impacting the convergence rate.

\begin{figure}[t]
\centering
\hspace{-3mm}
\subfigure[Training time with 14 workers]{
\includegraphics[width=0.21\textwidth]{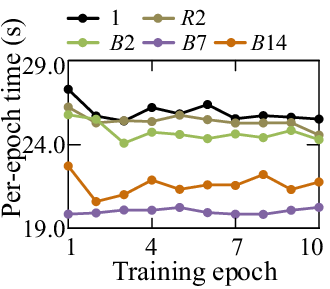}
}
\hspace{1.4mm}
\hspace{-0.5cm}
\subfigure[Training time with 12 workers]{
\includegraphics[width=0.21\textwidth]{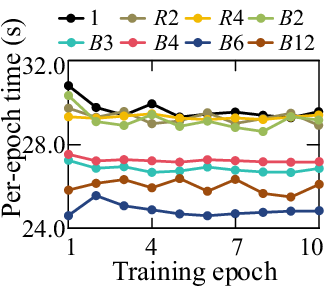}
}
\vspace{-2mm}
\caption{{Training time among different epochs}}
\vspace{-6mm}
\label{fig:adjustd}
\end{figure}

\vspace{-2mm}
\subsection{{Time-consuming stability of \textsf{SparDL} in different training epochs with different team number $d$.}}
\vspace{-1mm}
\label{sec: adjustd}
{
In this set of experiments, we evaluate the time-consuming stability of \textsf{SparDL} in different epochs with different team numbers $d$. The experimental results are shown in Fig.~\ref{fig:adjustd}. 
$Rx$, $Bx$ and 1 in the figure mean R-SAG with $d=x$, B-SAG with $d=x$ and \textsf{SparDL} without SAG, respectively.
From the experimental results, it is evident that the \textsf{SparDL} with the optimal $d$ is steadily faster than \textsf{SparDL} with other $d$.
Specifically, B-SAG ($d$ = 7) takes the least time in each of the first ten rounds in Fig.~\ref{fig:adjustd}(a), and B-SAG ($d$ = 6) takes the least time in each of the first ten rounds in Fig.~\ref{fig:adjustd}(b). They all have the optimal $d$ in their own experimental settings.
Therefore, users can choose the optimal $d$ by comparing the training time of \textsf{SparDL} with different $d$ in the first epoch.
}

\begin{figure}[t]
\centering
\hspace{-3mm}
\subfigure[VGG-16 on CIFAR-10]{
\includegraphics[width=0.21\textwidth]{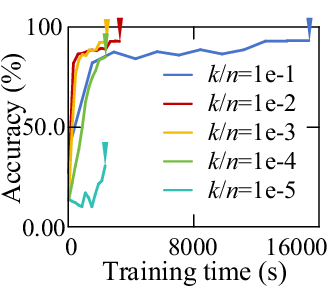}
}
\hspace{1.4mm}
\hspace{-0.5cm}
\subfigure[VGG-19 on CIFAR-100]{
\includegraphics[width=0.21\textwidth]{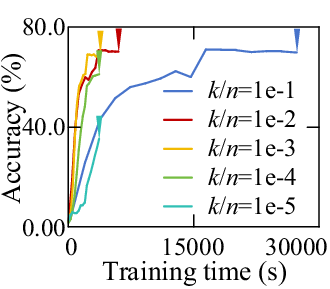}
}
\vspace{-2mm}
\caption{{\textsf{SparDL} with different $k$ for sparsification}}
\vspace{-6mm}
\label{fig:differentk}
\end{figure}

\vspace{-2mm}
\subsection{{Impact of the hyper-parameter $k$ for sparsification.}}
\vspace{-1mm}
{
In this set of experiments, we demonstrate the impact of $k$ for sparsification on the training efficiency and the convergence rate. The experimental results are illustrated in Fig.~\ref{fig:differentk}.
From this figure, we can observe that the smaller the ratio of $k$ to $n$, i.e., the number of dense gradients, the shorter the time it takes for \textsf{SparDL} to train the same number of epochs. However, the convergence rate gradually decreases. 
Specifically, for VGG-16 and VGG-19, after the $k/n$ is reduced from 1e-1 to 1e-2, the training time is greatly reduced (0.21$\times$ and 0.22$\times$), and the accuracy has no obvious change. After $k/n$ is reduced from 1e-2 to 1e-3, the training time is slightly reduced (0.75$\times$ and 0.70$\times$), and the accuracy is slightly reduced. However, after $k/n$ is reduced from 1e-3 to 1e-4 and 1e-5, the training time has hardly decreased (higher than 0.95$\times$), and the accuracy has decreased significantly, especially 1e-5. The reason why the training time is stable after $k/n$=1e-3 is that the communication time comes from latency cost and bandwidth cost. With the decrease of $k$, bandwidth cost will continue to decrease, but latency cost will remain unchanged, so the communication time remains stable after it is reduced to a certain extent. 
Therefore, the optimal $k$ should be chosen as $k/n$=1e-2 or 1e-3 when the communication time is low and maintains fast convergence rate.
}



\begin{figure}[t]
\centering
\hspace{-3mm}
\includegraphics[width=0.5\textwidth]{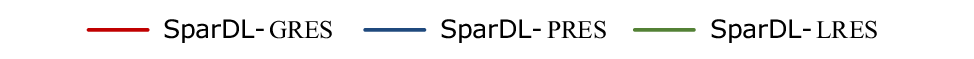}
\vspace{-5mm}

\centering
\vspace{-2mm}
\hspace{-3mm}
\subfigure[VGG-19, \textsf{SparDL}]{
\includegraphics[width=0.21\textwidth]{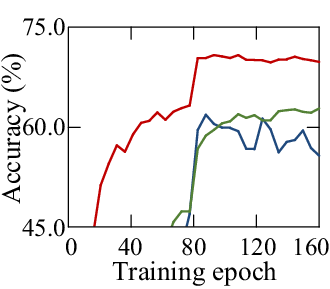}
}
\hspace{1.4mm}
\hspace{-0.5cm}
\subfigure[VGG-16, \textsf{SparDL}]{
\includegraphics[width=0.21\textwidth]{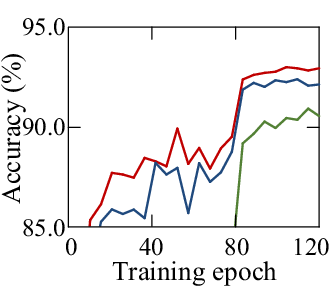}
}
\vspace{-3mm}

\centering
\hspace{-3mm}
\subfigure[VGG-16, \textsf{SparDL} (R-SAG)]{
\includegraphics[width=0.21\textwidth]{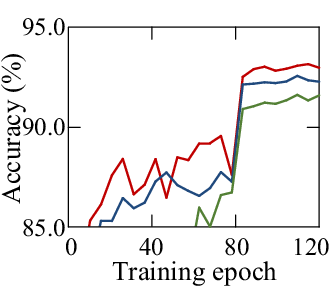}
}
\hspace{1.4mm}
\hspace{-0.5cm}
\subfigure[VGG-16, \textsf{SparDL} (B-SAG)]{
\includegraphics[width=0.21\textwidth]{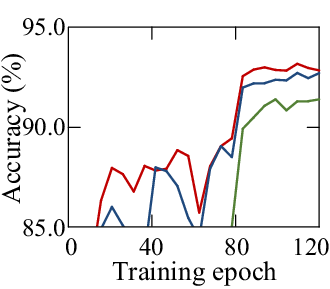}
}
\vspace{-2mm}
\caption{Convergence on \textsf{SparDL} using different residual collection algorithms with 14 workers}
\vspace{-6mm}
\label{fig:res}
\end{figure}

\vspace{-2mm}
\subsection{Impact of different residual collection algorithms.}
\vspace{-1mm}
In this set of experiments, we evaluate the effectiveness of the proposed global residual collection algorithm compared to existing residual collection algorithms. The experimental results are shown in Fig.~\ref{fig:res}.
It is observed from the results that the global residual collection algorithm maintains the fast convergence rate of \textsf{SparDL}.
Specifically, we compare the convergence rate of \textsf{SparDL} using three different residual collection algorithms: our global residual collection algorithm (\textsf{SparDL}-GRES), partial residual collection algorithm~\cite{DBLP:conf/icdcs/ShiWZTWHC19, DBLP:conf/ppopp/0002H22} (\textsf{SparDL}-PRES) and local residual collection algorithm~\cite{DBLP:conf/iclr/LinHM0D18} (\textsf{SparDL}-LRES). We utilize these communication methods
to train VGG-19 on CIFAR-100 and VGG-16 on CIFAR-10. Fig.~\ref{fig:res}(a) to Fig.~\ref{fig:res}(d) plot the test accuracy w.r.t. the training epoch under four different conditions. It is observed that \textsf{SparDL}-GRES consistently exhibits the most superior convergence rate, with the accuracy of \textsf{SparDL}-GRES consistently surpassing others after the 80th epoch when the learning rate is reduced. 

The reason for the distinction of \textsf{SparDL}-GRES is that the global residual collection algorithm accumulates all discarded gradients throughout the training process within the cluster, thereby maintaining the convergence rate. In contrast, \textsf{SparDL}-PRES and \textsf{SparDL}-LRES with partial residual collection algorithm and local residual collection algorithm overlook in-procedure residuals. However, such residuals, generated in substantial quantities within \textsf{SparDL}, play a crucial role in the training process. Neglecting to accumulate these residuals slows down the convergence rate of the training model. 
In conclusion, the global residual collection algorithm enables \textsf{SparDL} to preserve fast convergence rate when training different models with or without SAG methods.

\begin{figure}[t]
\centering
\hspace{-3mm}
\includegraphics[width=0.33\textwidth]{fig-icde/ICDE-E01-0.eps}

\centering
\vspace{-2mm}
\hspace{-3mm}
\subfigure[VGG-19 on CIFAR-100]{
\includegraphics[width=0.21\textwidth]{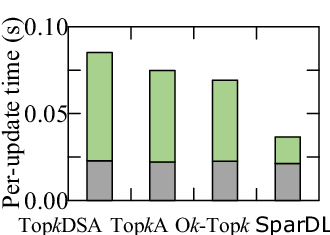}
}
\hspace{1.4mm}
\hspace{-0.5cm}
\subfigure[BERT on Wikipedia]{
\includegraphics[width=0.21\textwidth]{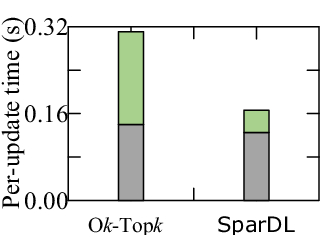}
}
\vspace{-2mm}
\caption{{Per-update time using RDMA network with 5 workers}}
\vspace{-6mm}
\label{fig:rdma}
\end{figure}

\vspace{-2mm}
\subsection{Performance with RDMA network.}
\vspace{-1mm}
{
To further evaluate the efficiency of \textsf{SparDL} with higher network bandwidth, we compared \textsf{SparDL} with baselines in a different GPU cluster on a small dataset CIFAR-100 using VGG-19 and a large dataset Wikipedia using BERT. There are five GPU machines in the cluster. Each machine is equipped with one NVIDIA A800 GPU, and all machines are connected to an InfiniBand network with RDMA service. 
The results are illustrated in Fig.~\ref{fig:rdma}, which verifies the superiority of our proposed \textsf{SparDL} even in the RDMA network. Fig.~\ref{fig:rdma}(a) and Fig.~\ref{fig:rdma}(b) depict the per-update time. For VGG-19, we can observe that \textsf{SparDL} achieves $4.0\times, 3.4\times$ and $3.0\times$ acceleration of communication cost compared to the baselines, respectively. For BERT, \textsf{SparDL} is $4.2\times$ faster than O$k$-Top$k$, i.e., the most efficient baseline, in terms of communication cost. 
Therefore, \textsf{SparDL} demonstrates efficiency when training deep learning models with high bandwidth networks.
}
\vspace{-1mm}
\section{Related Work}
\label{sec:related}
\vspace{-1mm}

All-Reduce operation~\cite{DBLP:conf/sc/SensiGA0H21, DBLP:conf/icdm/HakimiALS21, DBLP:conf/sigmod/MiaoNSYJM021, DBLP:journals/vldb/GuoZJWZCL21, DBLP:conf/hpca/DongCZYWFZLSPGJ20, DBLP:conf/icpp/ChuLASHEP17} is commonly used in data parallelism~\cite{DBLP:conf/osdi/LiAPSAJLSS14, DBLP:journals/pvldb/HuangJWCYYLGC18, DBLP:journals/pvldb/Renz-WielandGZM20, DBLP:conf/sosp/PengZCBYLWG19, DBLP:conf/sigmod/Renz-WielandGKM22, DBLP:conf/nips/ZhangCL15} distributed deep learning. 
Nonetheless, existing efficient All-Reduce methods~\cite{DBLP:conf/pvm/HoeflerGTT10, DBLP:journals/concurrency/ChanHPG07, mikami2018massively} are primarily designed for dense gradients, and are inefficient for synchronization with sparse gradients because of the sparse gradient accumulation (SGA) dilemma~\cite{DBLP:conf/ppopp/0002H22, DBLP:conf/icdcs/ShiWZTWHC19, DBLP:conf/sc/RenggliAAAH19}. 
Several sparse All-Reduce methods~\cite{DBLP:conf/icdcs/ShiWZTWHC19, DBLP:conf/ijcai/ShiZWTC19, DBLP:conf/ppopp/0002H22, DBLP:conf/sc/RenggliAAAH19} have been proposed to address the SGA dilemma. But they still suffer from low efficiency. 
Therefore, we aim to solve the SGA dilemma more efficiently to accelerate communication with sparse gradients.
\vspace{-1mm}
\section{{Discussion}}
\label{sec:lim}
\vspace{-1mm}

\noindent\textbf{{Limitations and Future Work.}}
(i) \textit{Heterogeneous environment.} \textsf{SparDL} tries to accelerate All-Reduce, which is mainly used in homogeneous environments. However, the heterogeneous environment also appears in real-world clusters, and there are some variants of All-Reduce proposed recently for this environment. In the future, we can extend SparDL to this environments.
(ii) \textit{Combining with quantization methods.} Sparsification and quantization are both common communication compression techniques. \textsf{SparDL} studies the sparsification and tries to solve the SGA dilemma efficiently. In the future, we may want to extend \textsf{SparDL} by combining quantization methods to further accelerate the communication.

\vspace{0.025in}
\noindent\textbf{{Practical Implications and Potential Applications.}}
Training deep learning models is usually time-consuming. 
Thus, it is important to use distributed deep learning (DDL) for efficient training. \textsf{SparDL} aims at accelerating the communication in DDL to further reduce the training time. And \textsf{SparDL} can be used in CV, NLP and other deep learning tasks.
In addition, a variety of large models have been proposed recently and these models are often trained by DDL. Because of the massive parameters of these models, 
it will lead to plenty of communication consumption. \textsf{SparDL} can be used in DDL and speed up the training of these large models by reducing the communication volume and efficient sparse communication.

\vspace{0.025in}
\noindent\textbf{{Relationship between \textsf{SparDL} and FSDP.}}
\textsf{SparDL} is orthogonal to FSDP~\cite{DBLP:journals/pvldb/ZhaoGVLHXWSOSDB23} or ZeRO-3~\cite{DBLP:conf/sc/RajbhandariRRH20}. The main purpose of FSDP and ZeRO-3 is to reduce the memory cost, so that larger models can be trained more easily. In these frameworks, the gradients are also synchronized by All-Reduce or Reduce-Scatter operations. Since these works communicate with dense gradients, \textsf{SparDL} is able to accelerate the communication of All-Reduce or Reduce-Scatter by using efficient sparse communication. Therefore, our \textsf{SparDL} can be combined into FSDP or ZeRO-3 frameworks.
\vspace{-1mm}
\section{Conclusions}
\label{sec:conclusion}
\vspace{-1mm}

In this paper, we analyze the low efficiency of existing sparse All-Reduce frameworks and propose \textsf{SparDL} to tackle these problems. For the first time, \textsf{SparDL} combines multiple selection processes and the Reduce-Scatter operation to deal with the Sparse Gradient Accumulation (SGA) dilemma. Besides, the \textsf{SparDL} uses the global residual collection algorithm to collect all discarded gradients in the cluster, ensuring fast training convergence. In addition, the Spar-All-Gather algorithm further improves the communication efficiency of \textsf{SparDL} and makes the ratio of latency and bandwidth cost adjustable. After conducting experiments on a wide range of common deep learning tasks, we observe that our \textsf{SparDL} achieves up to 4.9$\times$ speedup compared to the state-of-the-art methods while maintaining comparable effectiveness.

\vspace{-1mm}
\section*{Acknowledgment}
\vspace{-1mm}

This was supported  in part by the NSFC under Grants No. (62302436, U23A20296, 62025206), Ningbo Science and Technology Special Projects under Grant No. 2023Z212. Yuren Mao is the corresponding author of the work.
\balance


\balance

\bibliographystyle{abbrv}
\bibliography{ref}

\end{document}